\documentclass[letterpaper,11pt]{article}
\usepackage[margin=1in]{geometry}
\usepackage[bookmarks, colorlinks=true, plainpages = false, citecolor = blue,linkcolor=red,urlcolor = blue, filecolor = blue]{hyperref}
\usepackage{url}
\usepackage{amsmath,amsfonts,amsthm,amssymb,bm, verbatim,dsfont,mathtools}
\usepackage{color,graphicx,appendix}
\usepackage{subfigure}
\usepackage{etoolbox}
\usepackage{tikz}
\usepackage{xr,xspace}
\usepackage{todonotes}
\usepackage{paralist}
\usepackage{caption,soul}
\usepackage{algorithm}% http://ctan.org/pkg/algorithms
\usepackage{algorithmic}% http://ctan.org/pkg/algorithms
\makeatletter

\newtheorem{theorem}{Theorem}[section]
\newtheorem{lemma}[theorem]{Lemma}

\newtheorem{definition}[theorem]{Definition}

\newtheorem{conjecture}[theorem]{Conjecture}

\usepackage{xspace,prettyref}

\newcommand{\reals}{{\mathbb{R}}}

\newcommand{\naturals}{{\mathbb{N}}}

\newcommand{\eexp}{{\rm e}}

\newcommand{\diff}{{\rm d}}

\newcommand{\expect}[1]{\mathbb{E}\left[ #1 \right]}

\newcommand{\prob}[1]{ \mathbb{P}\left\{ #1 \right\} }

\newcommand{\var}{\mathsf{var}}

\newcommand{\Bern}{{\rm Bern}}
\newcommand{\Binom}{{\rm Binom}}
\newcommand{\Pois}{{\rm Pois}}

\newcommand{\eg}{e.g.\xspace}
\newcommand{\ie}{i.e.\xspace}

\newrefformat{eq}{(\ref{#1})}
\newrefformat{chap}{Chapter~\ref{#1}}
\newrefformat{sec}{Section~\ref{#1}}
\newrefformat{algo}{Algorithm~\ref{#1}}
\newrefformat{fig}{Fig.~\ref{#1}}
\newrefformat{tab}{Table~\ref{#1}}
\newrefformat{rmk}{Remark~\ref{#1}}
\newrefformat{clm}{Claim~\ref{#1}}
\newrefformat{def}{Definition~\ref{#1}}
\newrefformat{cor}{Corollary~\ref{#1}}
\newrefformat{lmm}{Lemma~\ref{#1}}
\newrefformat{prop}{Proposition~\ref{#1}}
\newrefformat{app}{Appendix~\ref{#1}}
\newrefformat{hyp}{Hypothesis~\ref{#1}}
\newrefformat{thm}{Theorem~\ref{#1}}

\newcommand{\indc}[1]{{\mathbf{1}_{\left\{{#1}\right\}}}}

\newcommand{\calE}{{\mathcal{E}}}

\newcommand{\calN}{{\mathcal{N}}}
\newcommand{\calO}{{\mathcal{O}}}

\newcommand{\calT}{{\mathcal{T}}}

\newcommand{\ER}{Erd\H{o}s-R\'enyi\xspace}

\renewcommand{\tilde}{\widetilde}
\renewcommand{\hat}{\widehat}

\begin{document}

\title{Density Evolution in the
Degree-correlated Stochastic Block Model}

\author{ 
Elchanan Mossel \and Jiaming Xu\thanks{
E. Mossel is with Department of Statistics, The Wharton School, University of Pennsylvania, Philadelphia, PA and with
the Departments of Statistics and Computer Science, U.C. Berkeley, Berkeley CA, \texttt{mossel@wharton.upenn.edu}.
J.~Xu is with the Simons Institute for the Theory of Computing, University of California, Berkeley, Berkeley, CA, 
\texttt{jiamingxu@berkeley.edu}.}
}

 % Use \Name{Author Name} to specify the name.
 % If the surname contains spaces, enclose the surname
 % in braces, e.g. \Name{John {Smith Jones}} similarly
 % if the name has a "von" part, e.g \Name{Jane {de Winter}}.
 % If the first letter in the forenames is a diacritic
 % enclose the diacritic in braces, e.g. \Name{{\'E}louise Smith}

 % Two authors with the same address
  % \coltauthor{\Name{Author Name1} \Email{abc@sample.com}\and
  %  \Name{Author Name2} \Email{xyz@sample.com}\\
  %  \addr Address}

 % Three or more authors with the same address:
 % \coltauthor{\Name{Author Name1} \Email{an1@sample.com}\\
 %  \Name{Author Name2} \Email{an2@sample.com}\\
 %  \Name{Author Name3} \Email{an3@sample.com}\\
 %  \addr Address}

\maketitle

\begin{abstract}
There is a recent surge of interest in identifying the sharp recovery thresholds for cluster recovery under the
stochastic block model. In this paper, we address the more refined question of how many vertices that will be misclassified
on average. We consider the binary form of the stochastic block model,
where $n$ vertices are partitioned into  two clusters
with edge probability
$a/n$ within the first cluster, $c/n$ within the second cluster, and
$b/n$ across clusters.  Suppose that as $n \to \infty$,
$a= b+ \mu \sqrt{ b} $,  $c=b+  \nu \sqrt{ b} $ for two fixed
constants $\mu, \nu$, and $b \to \infty$ with $b=n^{o(1)}$.
When the cluster sizes are balanced and $\mu \neq \nu$,
we show that
the minimum fraction of misclassified vertices on average is given by  $Q(\sqrt{v^*})$,
where $Q(x)$ is the Q-function for standard normal,
$v^*$ is the unique fixed point of $v= \frac{(\mu-\nu)^2}{16} + \frac{ (\mu+\nu)^2 }{16} \mathbb{E}[ \tanh(v+ \sqrt{v} Z)]$,
and $Z$ is standard normal. Moreover, the minimum  misclassified  fraction on average is attained by a local algorithm, namely
belief propagation, in time linear in the number of edges. Our proof techniques are based on connecting the cluster recovery problem to
tree reconstruction problems, and analyzing the density evolution of belief propagation on trees with Gaussian approximations.
\end{abstract}

\section{Introduction}
The problem of cluster recovery under the stochastic block model has
been intensely studied in statistics~\cite{HLL:1983,
SN:1997, BC:2009,cai2014robust,Zhangzhou15,Gao15}, computer science (where it is known as the planted partition
problem)~\cite{DF:1989,JS:1998,CK:2001,McSherry:2001,Coja-Oghlan05,Coja-oghlan10,Chen12, anandkumar2013tensormixed,ChenXu14}, and theoretical
statistical physics~\cite{DKMZ:2011, ZKRZ:2012, DKMZ:2011a}. In the simplest binary form, the stochastic block model assumes that
$n$ vertices are partitioned into two clusters with edge probability $a/n$ within the first cluster, $c/n$ within the second cluster, and $b/n$
across the two clusters. The goal is to reconstruct the underlying clusters from the observation of the graph.
Different reconstruction goals can be considered depending on how the model parameters $a,b,c$ scale with $n$ (See
\cite{AbbeSandon15} for more discussions):
\begin{itemize}
\item \emph{Exact recovery (strong consistency)}. If the average degree is $\Omega(\log n)$,
it is possible to exactly recover the clusters (up to a permutation of cluster indices) with high probability.
In the case with two equal-sized clusters, and  $a=c=\alpha \log n /n$ and $b=\beta \log n /n$ for two fixed $\alpha, \beta>0$,
 a sharp exact recovery threshold $(\sqrt{\alpha}-\sqrt{\beta})^2\ge 2$ has been found in \cite{Mossel14,Abbe14} and it is further shown that semi-definite programming can achieve the sharp threshold in \cite{HajekWuXuSDP14,Bandeira15}. The threshold for two unequal-sized clusters is proved in \cite{HajekWuXuSDP15}.
Exact recovery threshold with a fixed number of clusters  has been identified in \cite{HajekWuXuSDP15,YunProutiere14,ABBK}, and
more generally in \cite{AbbeSandon15,PerryWein15} with heterogeneous cluster sizes and edge probabilities.

\item \emph{Weak recovery (weak consistency)}. If the average degree is $\Omega(1)$, one can hope for
misclassifying  only $o(n)$ vertices  with high probability, which is
 known as weak recovery or weak consistency.
 In the setting with two approximately equal-sized clusters and $a=c$, it is shown in \cite{yun2014adaptive,Mossel14} that weak recovery is possible if and only $(a-b)^2/(a+b) \to \infty.$

\item \emph{Correlated recovery (non-trivial detection)}. If the average degree is $\Theta(1)$,
exact recovery or weak recovery becomes hopeless as
the resulting graph under the stochastic block model will have at least a constant fraction of
isolated vertices.
 Moreover, it is easy to see that even
vertices with constant degree cannot be labeled accurately given all the other
vertices' labels are revealed.
Thus one goal in the sparse graph regime is to find a partition positively correlated with the true one (up to a permutation
of cluster indices), which is also called non-trivial detection. In the setting with two approximately equal-sized clusters and $a=c$,
it was first conjectured in~\cite{DKMZ:2011} and
later proven in~\cite{MONESL:15,MNS:2013b,Massoulie:2013} that correlated recovery is feasible if and only if $(a-b)^2>2(a+b)$.
A spectral method based on the non-backtracking matrix is recently shown to achieve the sharp threshold in \cite{BordenaveLelargeMassoulie:2015dq}.
\end{itemize}

In practice, one may be interested in the finer question of how many vertices that will be misclassified on expectation or with high probability.
In the two equal-sized clusters setting, previous results on exact recovery, weak recovery, and correlated recovery
 provide conditions under which the minimum fraction of misclassified vertices on average
is $o(1/n)$, $o(1)$, and strictly smaller than $1/2$, respectively.
By assuming $(a-b)^2/(a+b) \to \infty$, recent work \cite{Zhangzhou15,Gao15} showd that the expected misclassified fraction decays to zero exponentially fast and gives a sharp characterization of the decay exponent under a minimax framework. However, all these previous results do not shed light on the important question of when it is possible to misclassify only $\epsilon$ fraction of vertices on
expectation, for any finite $\epsilon \in (0,1/2).$
To the best of our knowledge, it is an open problem
to find a closed-form expression of the expected misclassified fraction in terms of the model parameters. In this paper, we give such a simple formula
in the special case of two approximately equal-sized clusters.
Specifically, suppose that
\begin{align}
a=b+ \sqrt{b} \mu, \qquad c=b+ \sqrt{b} \nu, \qquad
b \to \infty, \qquad b =n^{o(1)},
\label{eq:asymptotics}
\end{align}
for two fixed constants $\mu, \nu$. We further assume that
$\mu \neq \nu$ so that the vertex degrees are statistically correlated with the cluster structure,
and hence the name of the \emph{degree-correlated stochastic block model}.
We show that the minimum  fraction of misclassified vertices on average is given by $Q(\sqrt{v^*})$,
where $Q(x)$ is the Q-function for standard normal,
$v^\ast$ is the unique fixed point of $v= \frac{(\mu-\nu)^2}{16} + \frac{ (\mu+\nu)^2 }{16} \expect{ \tanh(v+ \sqrt{v} Z)}$,
and $Z$ is standard normal. Moreover,
the minimum expected misclassified fraction can be attained by a local algorithm, namely
belief propagation (BP) algorithm (See Algorithm \prettyref{alg:MP_commun}), in time $O(n b^2)$.  The local belief propagation algorithm can be viewed as an iterative algorithm
which improves on the misclassified fraction on average step by step; running  belief propagation for one iteration reduces to the simple
thresholding algorithm based on vertex degrees.
It is crucial to assume $\mu \neq \nu$ for the above results to hold,
otherwise it is well-known  (see e.g.~\cite{KaMoSc:14}) that no local algorithm
can even achieve the non-trivial detection. Nevertheless, under a slightly stronger assumption that $b \to \infty$ and $b=o(\log n)$,
we show that if $\mu =\nu$ with $|\mu|>2$,
local belief propagation combined with a global algorithm capable of non-trivial detection when $|\mu|>2$,
attains the minimum expected misclassified fraction $Q(\sqrt{\overline{v} })$ in polynomial-time, where $\overline{v}$
is the largest fixed point of $v=\frac{ \mu^2 }{4} \expect{ \tanh(v+ \sqrt{v} Z)}.$

When the clusters sizes are unbalanced, \ie, one cluster is of size approximately $\rho n$ for $\rho \in (0,1/2)$, 
we give a lower bound on the minimum expected misclassified fraction, and an upper bound
 attained by the local belief propagation algorithm. However, we are unable to prove that the upper bound matches the lower bound.
 In fact, numerical experiments suggest that there exists a gap between the upper and lower bound when the cluster sizes are very
 unbalanced, \ie, $\rho$ is close to $0$.

Our proofs are mainly based on two useful techniques introduced in previous work. First, in the regime \prettyref{eq:asymptotics}, the observed graph
is locally tree-like, so we connect the cluster recovery problem to reconstruction problems on trees, and for the tree problems, the optimal estimator can be computed by
belief propagation algorithm. Such connection has been investigated before in \cite{MONESL:15, MNS:2013a, MosselXu15}.
Second, we characterize the  density evolution of belief propagation on trees with Gaussian approximations, and
as a result, we get a recursion with the largest fixed point
corresponding to a lower bound on the minimum expected misclassified fraction, and
the smallest fixed point correspond to the expected misclassified fraction attained by the local belief propagation  algorithm.
Density evolution  has been widely used for the analysis of multiuser detection \cite{MontanariTse06} and sparse
graph codes \cite{Urbanke08,Mezard09}, and more recently
has been introduced for the analysis of finding a single community in a sparse graph \cite{Montanari:15OneComm}.
As a final piece, we prove that in the balanced cluster case, the recursion has a unique fixed point
using the ideas of symmetric random variables   \cite{Urbanke08,Montanari05} and the first-order stochastic dominance,
thus establishing the tightness of the lower bound and the optimality of the local BP simultaneously.

We  point out that  local algorithm by itself is a thriving research area  (see~\cite{LyonsNazarov:11,HaLoSz:12,GamarnikSudan:14} and the references therein).
Intuitively, local algorithms are one type of algorithms that make decision for each vertex
just based  on the neighborhood of small radius around the vertex; these algorithms are by design easy to run in a distributed fashion. Under the context of community detection,
local algorithms determine which community each vertex lies in just based on  the local neighborhood around each vertex (see \cite{Montanari:15OneComm} for a formal
definition). Recent work \cite{MosselXu15} shows that with the aid of extra noisy label information on cluster structure,
the local algorithms can be optimal in minimizing the expected misclassified fraction in the stochastic block model.
In comparison, we show that when the vertex degrees are correlated with the cluster structure,
the local algorithms can be optimal even without the extra noisy label information.

In closing, we compare our results with the recent results in~\cite{Montanari:15OneComm,HajekWuXu_one_beyond_spectral15}, which studied the problem of
finding a single community of size $\rho n$ in a sparse graph. When $\nu=0$, \ie $b=c$, the stochastic block model considered in this paper, specializes to
the single community model studied in~\cite{Montanari:15OneComm}, and the recursion  of density evolution derived in this paper
reduces to the recursion derived in \cite[Eq. (36)]{Montanari:15OneComm}. It is shown in \cite{Montanari:15OneComm,HajekWuXu_one_beyond_spectral15} that
the local algorithm is strictly suboptimal comparing to the global exhaustive search when $\rho \to 0$. In contrast, we show that
if $\rho=1/2$, the local algorithm is optimal in minimizing the expected fraction of misclassified vertices as long as $\mu \neq \nu$,
and give a sharp characterization of the minimum expected misclassified fraction.

\paragraph*{Parallel Independent Work}
The problem of cluster recovery
under the degree-correlated stochastic block model with multiple clusters was
independently studied in \cite{ZhangMooreNewman15}.
Based on the cavity method and numerical simulations, it
is shown that with at most four clusters of unequal sizes but same in and out degrees,
the non-trivial detection threshold phenomenon disappears, making
the minimum fraction of misclassified vertices on average a continuous function of model
parameters. In comparison, in the regime \prettyref{eq:asymptotics} with two equal-sized clusters and $\mu \neq \nu$,
we give a more precise answer, showing that the fraction of misclassified vertices on average is
$Q(\sqrt{v^\ast})$, where $v^\ast$ is the unique fixed point of $v= \frac{(\mu-\nu)^2}{16} + \frac{ (\mu+\nu)^2 }{16} \expect{ \tanh(v+ \sqrt{v} Z)}$.
Moreover, it is shown in \cite{ZhangMooreNewman15} that with more than four clusters of unequal sizes, there exists a regime where
two stable fixed points coexist, with the smaller one corresponding to the performance of
local belief propagation, and the larger one corresponding to the performance of belief propagation
initialized based on the true cluster structure. We find that the same phenomenon also happens
in the case of two clusters with very unbalanced sizes and different in and out degrees (See  \prettyref{sec:open} for
details).

We recently became aware of the work \cite{DeshapandeAbbeMontanari15},
who studied the problem of cluster recovery under the
stochastic block model in the symmetric setting with two equal-sized clusters and $a=c$.
By assuming that $\frac{(a-b)^2}{2(a+b) (1- (a+b)/2n)} \to \mu$ for a fixed constant $\mu$ and
$(a+b)(1-(a+b)/2n) \to \infty$,
a sharp characterization of the per-vertex mutual information between the vertex labels and
the graph is given in terms of $\mu$ and $\overline{v}$,
where $\overline{v}$ is the largest fixed point  of  $v=\frac{ \mu^2 }{4} \expect{ \tanh(v+ \sqrt{v} Z)}$.
In comparison, we show that
the minimum fraction of vertices misclassified on expectation is given by $Q(\sqrt{\overline{v}})$ and it is attainable
in polynomial time with an additional technical assumption that $b=o(\log n)$.
Interestingly, the point (a) of Lemma 6.1 in \cite{DeshapandeAbbeMontanari15} is a special case of
 \prettyref{lmm:hmonotone} with $\rho=1/2$ in our paper.
The proof of  Lemma 6.1 given in \cite{DeshapandeAbbeMontanari15} and the proof of  \prettyref{lmm:hmonotone}  given in this paper
are similar: both used the ideas of symmetric random variables   \cite{Urbanke08,Montanari05}. One slight difference is that  to prove the concavity of the mapping in the recursion when $\rho=1/2$, we used the first-order stochastic dominance, while \cite{DeshapandeAbbeMontanari15} computes the second-order derivative.

\section{Model and Main Results}

We consider the binary stochastic block model with $n$ vertices partitioned into two clusters, where
each vertex is independently assigned into the first cluster  with probability $\rho \in (0,1)$ and the second cluster
with probability $\bar \rho \triangleq 1-\rho $.\footnote{Notice that the cluster sizes are random and concentrate on $\rho n$ and $(1-\rho) n$. 
A slightly different model assumes that the vertices are partitioned into two clusters of deterministic sizes, exactly given by $\rho n$ and $(1-\rho )n$. 
The two models behave similarly, but for ease of analysis, we focus on the random cluster size model in this paper.} Each pair of vertices is connected independently with probability $a/n$ if two
vertices are in the first cluster, with probability $c/n$ if they are in the second cluster, and
with probability $b/n$ if they are in two different clusters.  Let $G=(V,E)$ denote the observed graph and  
$A$ denote the adjacency matrix of the graph $G$. Let $\sigma$
denote the underlying vertex labeling such that $\sigma_i=+$ if vertex $i$ is in the first cluster and $\sigma_i=- $ otherwise.
The model parameters
$\{\rho, a, b, c\}$ are assumed to be known, and the goal is to estimate the vertex labeling $\sigma$ from the observation of $G. $
More precisely, we have the following definition.
\begin{definition}
The {\em reconstruction problem on the graph} is the  problem of inferring $\sigma$
 from the observation of $G$. The expected fraction of vertices  misclassified by an estimator $\hat{\sigma}$ is given by
\begin{align}
p_{G}(\hat{\sigma})= \frac{1}{n}  \sum_{i=1}^n \prob{\sigma_i \neq \hat{\sigma}_i}. \label{eq:estimationaccuracy}
\end{align}
Let $p^\ast_{G}$  denote the minimum expected misclassified fraction among all
possible estimators based on $G$.
\end{definition}
The optimal estimator in minimizing the error probability $ \prob{\sigma_i \neq \hat{\sigma}_i}$ is the maximum a posterior (MAP) estimator, which is given by
$
 2 \times  \indc{ \prob{\sigma_i=+  | G } \ge  \prob{\sigma_i=-  |G} } -1,
$
and the minimum error probability is given by $ \frac{1}{2}- \frac{1}{2} \expect{| \prob{\sigma_i=+  | G} -  \prob{\sigma_i=-  |G } |}$.
Hence, the minimum expected misclassified fraction $p^\ast_{G}$ is given by
\begin{align}
p_{G} ^\ast  & =  \frac{1}{2} - \frac{1}{2n}  \sum_{i=1}^n \expect{ \big| \prob{\sigma_i=+  | G } -  \prob{\sigma_i=-  |G } \big| } \nonumber  \\
&=  \frac{1}{2} - \frac{1}{2} \expect{ | \prob{\sigma_i=+  | G } -  \prob{\sigma_i=-  |G } |}, \label{eq:optimalaccuracygraph}
\end{align}
where the second equality holds due to the symmetry among vertices. In the special case with $\rho=1/2$ and $a=c$, 
 the two clusters are symmetric; thus $p_{G} ^\ast=1/2$ and one can only hope to estimate $\sigma$ up to a global flip of 
 sign. In general, computing the MAP estimator is computationally intractable and it is unclear
whether the minimum expected misclassified fraction  $p^\ast_{G}$ can be achieved by some estimator computable in polynomial-time.

Throughout this paper, we assume that $\rho$ is fixed and focus on the regime \prettyref{eq:asymptotics}. As the average degree is $n^{o(1)}$, it is well-known that a local neighborhood of a vertex is a tree with high probability. Thus, it is natural
to study the local algorithms. More precisely, we consider a local belief propagation algorithm to approximate the MAP estimator in
the next subsection.

\subsection{Local Belief Propagation Algorithm}
Our local belief propagation algorithm is given in Algorithm \prettyref{alg:MP_commun}. Specifically, let $\partial i$ denote the set of neighbors of $i$,
and $F(x) = \frac{1}{2} \log \left(  \frac{ \eexp^{2x } \rho a   + \bar \rho b  }{ \eexp^{2x } \rho b +\bar \rho c} \right)$.
Let $d_+= \rho a+ \bar \rho b$ and $d_-= \rho b + \bar \rho c$ denote the expected vertex degree in the first and second cluster,
respectively. Define
the message transmitted from vertex $i$ to vertex $j$ at $t$-th iteration as
\begin{align}
R_{i \to j}^t  = \frac{- d_+ + d_- }{2}  + \sum_{\ell \in \partial i \backslash \{j\} }  F (R^{t-1}_{\ell \to i} ), \label{eq:mp_commun}
\end{align}
with initial conditions $R_{i \to j}^{0} =0$ for all $i \in [n]$ and $j \in \partial i$.
Then we  define the belief of vertex $u$ at $t$-th iteration $R_u^t$ to be
\begin{align}
R_{u}^t  = \frac{- d_+ + d_- }{2}  + \sum_{\ell \in \partial u }  F(R^{t-1}_{\ell \to u} ) \label{eq:mp_combine_commun}.
\end{align}
\begin{algorithm}[htb]
\caption{Belief propagation for cluster recovery}\label{alg:MP_commun}
\begin{algorithmic}[1]
\STATE Input: $n \in \naturals,$ $\rho \in (0,1)$, $a/b, c/b$,  adjacency matrix $A \in \{0,1\}^{n\times n}$,  and $t \in \naturals$.
\STATE Initialize: Set  $R^{0}_{i \to j}=0$ for all $i \in [n]$ and $ j \in \partial i$.
\STATE Run $t-1$ iterations of message passing as in \prettyref{eq:mp_commun} to compute $R^{t -1 }_{i\to j}$ for all $i \in [n]$ and $j \in \partial i$.
\STATE Compute  $R_{i}^{t}$ for all $i \in [n]$ as per \prettyref{eq:mp_combine_commun}.
\STATE Return $\hat{\sigma}_{\rm BP}^t$ with $\hat{\sigma}^t_{\rm BP} (i)= 2 \times \indc{ R^{t}_i \ge -\varphi }  -1,$ where $\varphi= \frac{1}{2} \log \frac{\rho}{1-\rho}$.
\end{algorithmic}
\end{algorithm}

As we will show in \prettyref{sec:connectiongraphtree}, the message passing  as in \prettyref{eq:mp_commun} and \prettyref{eq:mp_combine_commun}
exactly computes the log likelihood ratio for a problem of inferring $\sigma_u$ on a suitably defined tree model with root $u$.
Moreover, in the regime \prettyref{eq:asymptotics}, there exists a coupling such that the local neighborhood of a fixed vertex $u$ is
the same as the tree model rooted at $u$ with high probability.
These two observations together
suggest that $R_{u}^t $ is a good approximation of $\frac{1}{2} \log \frac{\prob{G| \sigma_u=+} }{\prob{G | \sigma_u=-}}$, and thus
we can estimate $\sigma_u$ by truncating $R_u^t$ at the optimal threshold $-\varphi=\frac{1}{2} \log \frac{1-\rho}{\rho}$, according to the MAP rule.

We can see from \prettyref{eq:mp_commun} that in each BP iterations,
each vertex  $i$ needs to compute $|\partial i|$ outgoing messages.
To this end, $i$ can first compute $R_{i}^t$ according to \prettyref{eq:mp_combine_commun},
and then subtract $F(R_{j \to i}^{t-1} )$ from $R_i^t$ to get $R_{i \to j}^t$
for every neighbor $j$ of $i$. In this way,
each BP iteration runs in time $O( m )$, where $m$ is the total number of edges.
Hence $\hat{\sigma}_{\rm BP}^t$ can be computed in time $O (t m).$

Finally, notice that  Algorithm \prettyref{alg:MP_commun} needs to know the parameters $\{\rho, a/b, c/b\}$.
For the main results of this paper continue to hold,  the values of the parameters are only needed to know up to $o(1)$ additive errors.   
In fact, there exists a fully data-driven procedure to consistently estimate those parameters, see \eg, \cite{HajekWuXuSDP15}[Appendix B].

\subsection{Main Results}\label{sec:mainresults}
The following theorem characterizes the expected fraction of vertices misclassified by $\hat{\sigma}_{\rm BP}^t$ as $n \to \infty$;
it also gives a lower bound on the minimum expected misclassified fraction as $n \to \infty$. Furthermore, in the case $\rho=1/2$ and $\mu \neq \nu$, $\hat{\sigma}_{\rm BP}^t$  achieves the lower bound as $t \to \infty$ after $n \to \infty.$

\begin{theorem}\label{thm:optimalaccuracy}
Assume $\rho \in (0,1)$ is fixed and consider the regime \prettyref{eq:asymptotics}.
 Let $$
 h(v) = \expect{ \tanh ( v + \sqrt{v} Z + \varphi )   },
 $$
 where $Z \sim \calN(0,1)$ and $\varphi= \frac{1}{2} \log \frac{\rho}{1-\rho}$.
Let $\lambda=  \frac{\rho( \mu+\nu)^2}{8} $ and $\theta= \frac{\rho(\mu-\nu)^2 }{8} +  \frac{(1 - 2\rho ) \nu^2}{4}$.
 Define $\underline{v}$ and $\overline{v}$ to be the smallest and largest fixed point of
 $$v= \theta+ \lambda h(v),$$
 respectively\footnote{The existence of fixed points of $v \mapsto  \theta+ \lambda h(v)$ follows from
 Brouwer's fixed-point theorem and the fact that $h(v) \le 1$. }. Define $(v_t: t \ge 0)$ recursively by $v_0=0$ and $
v_{t+1} = \theta+ \lambda h(v_t).$
 Let $\hat{\sigma}_{\rm BP}^t$ denote the estimator given by Belief Propagation applied for $t$ iterations, as defined in Algorithm \prettyref{alg:MP_commun}.
  Then $\lim_{t \to \infty} v_t = \underline{v},$ $(\rho \mu - \bar \rho \nu )^2/4   \le \underline{v}\le \overline{v}  \le (\rho \mu^2 + \bar \rho \nu^2)/4,$
  and
 \begin{align*}
\lim_{n \to \infty} p_{G } (\hat{\sigma}_{\rm BP}^t ) & = \rho Q \left(\frac{ v_t  + \varphi }{ \sqrt{v_t } } \right) + (1-\rho) Q \left(\frac{ v_t - \varphi }{ \sqrt{v_t } } \right) ,  \\
 \liminf_{n \to \infty}p_{G}^\ast & \ge \rho Q \left(\frac{ \overline{v}  + \varphi }{ \sqrt{\overline{v}  } } \right) +  (1-\rho) Q \left(\frac{ \overline{v}  - \varphi }{ \sqrt{\overline{v}  } } \right) ,
 \end{align*}
 where $Q(x)=\int_{x}^{+\infty} \frac{1}{\sqrt{2\pi}} \eexp^{-y^2/2} \diff y$.
Moreover, if $\rho=1/2$ and $\mu \neq \nu$, then $\underline{v}=\overline{v}= v^\ast$,
 and thus
 \begin{align*}
 \lim_{t \to \infty} \lim_{n \to \infty} p_{G}(\hat{\sigma}_{\rm BP}^t )= \lim_{n \to \infty}p_{G}^\ast =  Q(\sqrt{ v^\ast }),
 \end{align*}
 where $v^\ast$ is the unique fixed point of  $v= \frac{(\mu-\nu)^2}{16} + \frac{ (\mu+\nu)^2 }{16} \expect{ \tanh(v+ \sqrt{v} Z)}.$
\end{theorem}

If $\rho \mu \neq \bar \rho \nu$ so that the vertex degrees are statistically correlated with the cluster structure,
we have $\underline{v}>0$ and thus  
$\lim_{t \to \infty} \lim_{n \to \infty} p_{G}(\hat{\sigma}_{\rm BP}^t ) \ge \min\{ \rho, 1-\rho\}$.
Hence, the local application of BP strictly outperforms the trivial estimator, which always guesses the label
of all vertices to be $+1$ if $\rho\ge 1/2$ and $-1$ if $\rho<1/2$.
In the balanced case $\rho=1/2$, the local BP achieves the minimum expected misclassified fraction. Numerical experiments
further indicate that the local BP is still optimal in the unbalanced case provided that $\rho$ is not close to $0$ or $1$; however,
we do not have a proof (See \prettyref{sec:open} for more discussions).

If $\rho \mu=\bar \rho \nu$, then $\underline{v} =0$ and thus  
$$
\rho Q \left(\frac{ \underline{v}  + \varphi }{ \sqrt{\underline{v}  } } \right) + (1-\rho) Q \left(\frac{ \underline{v}  -\varphi }{ \sqrt{\underline{v}  } } \right) = \min\{ \rho, 1-\rho\}.
$$
In this case, our local application of BP  cannot do better than the trivial estimator.
In fact, the local neighborhoods are statistically uncorrelated
with the cluster structure, and one can further argue that no local algorithm  can achieve non-trivial detection (see e.g.~\cite{KaMoSc:14}). Although local algorithms are bound to fail, there might still exist some efficient global algorithms which achieve the minimum expected misclassified fraction. The following theorem shows that this is indeed the case when $\rho =1/2$, $\mu=\nu$ and $b=o(\log n)$.

\begin{theorem}\label{thm:symmetry}
Assume $\rho=1/2$, $a=c=b+\sqrt{b} \mu $ for some fixed constant $\mu$, and $b \to \infty$ such that $b=o(\log n)$.
For an estimator $\hat{\sigma}$ based on graph $G$, define the fraction of vertices misclassified by $\hat{\sigma}$ as
\begin{align}
\calO(\hat{\sigma}, \sigma)  = \frac{1}{n} \min \left\{   \sum_{i=1}^n \indc{ \sigma_i \neq \hat{\sigma}_i }, \;  \sum_{i=1}^n \indc{ \sigma_i \neq  -\hat{\sigma}_i } \right\}. \label{eq:defoverlap}
\end{align}
If $|\mu|>2$, then
\begin{align}
\lim_{n \to \infty} \inf_{\hat{\sigma}}   \expect{  \calO(\hat{\sigma}, \sigma) }= Q \left(\sqrt{ \overline{v} } \right), \label{eq:correlationmin}
\end{align}
where the infimum ranges over all possible estimators $\hat{\sigma}$ based on graph $G$; $\overline{v} >0$ is the largest fixed point of $v= \frac{ \mu^2 }{4} \expect{ \tanh(v+ \sqrt{v} Z)}.$ Moreover, there is a polynomial-time estimator such
that for every $\epsilon>0$, it misclassifies at most $Q \left( \sqrt{ \overline{v} }  \right)-\epsilon$ fraction of vertices on expectation.
\end{theorem}

In contrast to \prettyref{eq:estimationaccuracy}, the fraction of vertices misclassified by $\hat{\sigma}$ is defined up to
a global flip of signs of  $\hat{\sigma}$ in \prettyref{eq:defoverlap}. This is because in the case $\rho=1/2$ and $a=c$,  due to the symmetry between $+$ and $-$, $\sigma$ and $-\sigma$ have the same distribution conditional on graph $G$.
Thus, it is impossible to reliably estimate the sign of $\sigma$ based on graph $G$.

Note that $|\mu|=2$ corresponds to the Kesten-Stigum bound \cite{KestenStigum:66}.  It is shown in \cite{MONESL:15} that
if $|\mu|<2$, correlated recovery is impossible and thus the minimum expected misclassified fraction is $0$;
Remarkably, \cite{Massoulie:2013,MNS:2013b,BordenaveLelargeMassoulie:2015dq} prove that correlated recovery is efficiently
achievable if $|\mu|>2$. Our results further show that in this case with $b \to \infty$ and $b=o(\log n)$,
the minimum  expected misclassified fraction is
$Q \left(\sqrt{ \overline{v} } \right)$ and it can also be attained in polynomial-time. The proof of \prettyref{thm:symmetry}
is mainly based on two observations. First, it is shown in  \cite{MNS:2013a} that the local BP is able to improve a clustering that is
slightly better than a random guess to achieve the minimum expected misclassified fraction if $|\mu|>C$ for a universal constant $C>0$.
Second, we find that if $|\mu|>2$, the recursion $v= \frac{ \mu^2 }{4} \expect{ \tanh(v+ \sqrt{v} Z)}$ derived in the density evolution
analysis of local BP has only two fixed points: $0$ and $\overline{v}>0$,
where $0$ is unstable and $\overline{v}$ is stable. These two results together establish that
if $|\mu|>2$, then running  the local BP for $t$ iterations with a correlated initialization provided by a non-trivial detection algorithm is able to attain the minimum expected misclassified fraction $Q \left(\sqrt{ \overline{v} } \right)$ as $t \to \infty$.

\subsection{Proof Ideas}
The proof of \prettyref{thm:optimalaccuracy} is based on two useful tools. First, we connect the cluster recovery problem
to the reconstruction problem on trees.
%showing that the minimum expected misclassified fraction is lower bounded by
%the error probability of inferring the root label given the tree structure and the labels of leaves on a tree model
%and the expected fraction of vertices misclassified by the local BP algorithm is approximately the error probability of inferring the root label given only
%the tree structure.
Second, we use the density evolution with Gaussian approximations to give a sharp characterization of
error probabilities of tree reconstruction problems,
in terms of fixed points of a recursion.
%Finally, we show that in the balanced cluster case,
% the recursion has the unique fixed point and thus the two error probabilities coincide.

%density evolution \cite{Urbanke08,Mezard09},
%which was recently applied for analyzing a similar problem of finding a single community in a sparse graph \cite{Montanari:15OneComm}.
\begin{itemize}
\item To bound from below the minimum expected misclassified fraction, we bound the error probability of inferring $\sigma_u$ for a specific vertex $u$. Following~\cite{MONESL:15}, we consider an oracle estimator, which in addition to the graph structure, the exact labels of all vertices at distance exactly $t$ from $u$ are also revealed. As in~\cite{MONESL:15}, it is possible to show that the best estimator in this case is given by BP for $t$ levels initialized using the exact labels at distance $t$. In contrast, the expected fraction of vertices misclassified by the local BP algorithm
     approximately equals to the error probability of inferring $\sigma_u$ solely based on  the neighborhood of vertex $u$ of radius $t$, without having access to the exact labels of vertices at distance $t$.
%\item In the regime \prettyref{eq:asymptotics}, the neighborhood of a vertex $u$ is locally tree-like. Thus the local BP algorithm minimizes the error probability of inferring $\sigma_u$ given the neighborhood of vertex $u$ of radius $t$. The only difference between the local BP and BP lower bound is that the local BP does not have access to the exact labels of vertices at distance $t$.

\item We characterize the density evolution of the local BP and the BP lower bound using Gaussian approximations, and get a recursion with the largest fixed
point corresponding to the BP lower bound, and the smallest fixed point corresponding to the expected fraction of vertices misclassified by the local BP as
$ t \to \infty.$ In the balanced cluster cases, we further show that there is a unique fixed point for the recursion, and thus the BP lower bound matches the expected fraction of vertices misclassified by the local BP.

%
%Since the  neighborhood of a vertex $u$ is locally tree-like, in \prettyref{eq:mp_commun},
%the incoming messages to vertex $u$ from its neighbors are independent, and the
%outgoing messages from $u$ is a sum of independent random variables.
%Because the mean vertex degrees diverge, due to the central limit theorem, the distribution of the outgoing message from $u$ conditioning on the label of $u$
% converges to Gaussian.
%
% Moreover, $v_t$ admit a simple recursion over $t$ as v_{t+1} = \theta+ \lambda h(v_t).$

%\item In the regime \prettyref{eq:asymptotics}, the neighborhood of a vertex $u$ is locally tree-like and thus the incoming messages to vertex $u$ from its neighbors are independent in BP iterations. Since the expected vertex degrees diverge, the  distribution of the sum of incoming messages conditioning on the label of $u$ converges to Gaussian. Moreover, the mean and variance admit a simple recursion over $t$, which converge to a fixed point as $t\to \infty$.
%\item
% The only difference between our application of BP and the BP upper bound is the quality of information at distance exactly $t$ from vertex $u$. Hence, the mean and variance for
%both BPs satisfy the same recursion but with different initialization. If there is a unique fixed point of the recursion for mean and variance, then the mean and variance for
%both BPs  converge to the same values as $t\to \infty$.
\end{itemize}

%\subsection{Comparisons to previous work} \label{sec:comparison}

 \subsection{Numerical Experiments and Open Problems}\label{sec:open}

 \begin{figure}[ht]
\centering
\includegraphics[width=3.5in]{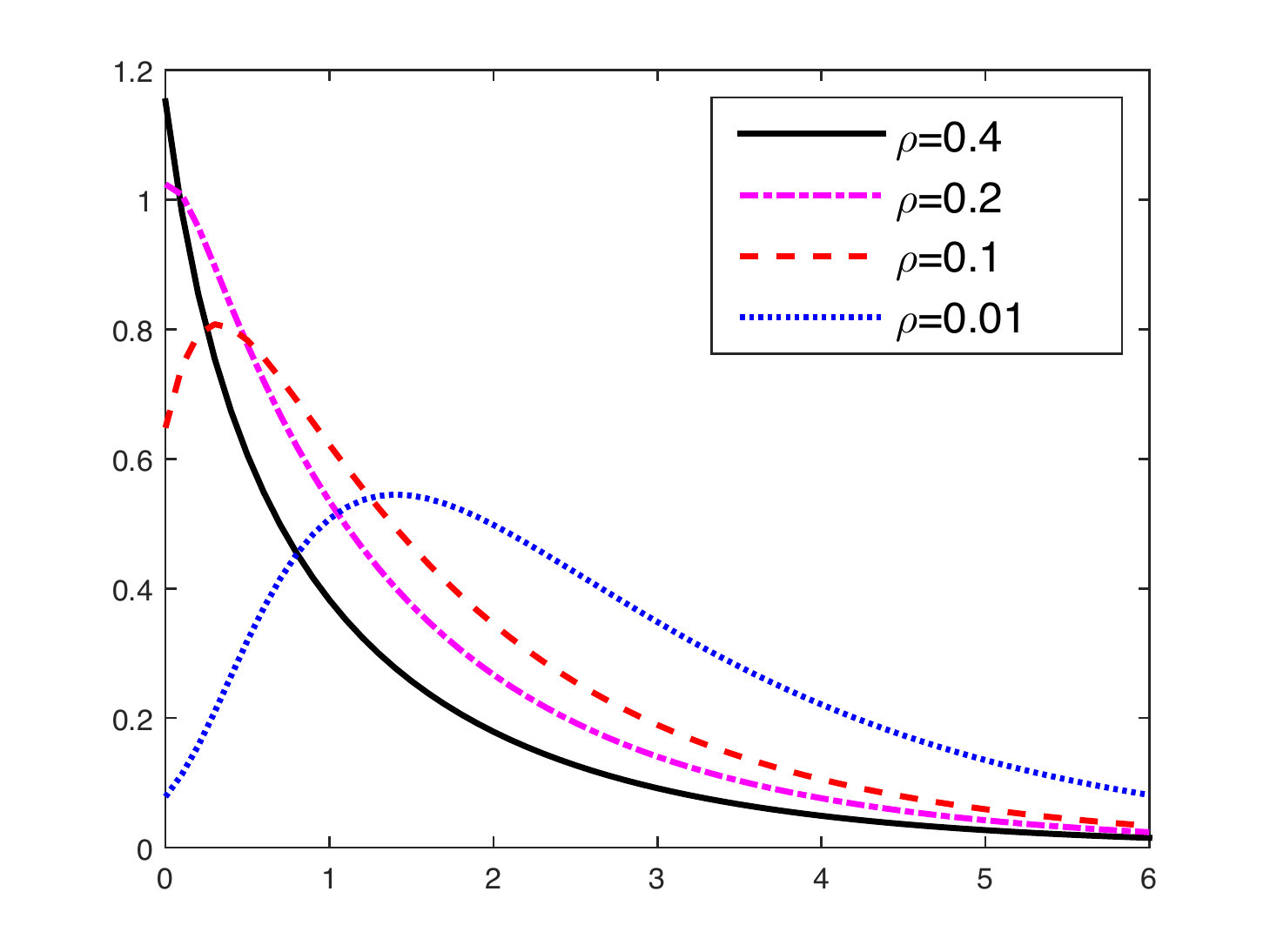}
\caption{Numerical calculations of $h'(v)$ (y axis) versus $v \in [0, 6]$ (x axis) with different $\rho.$ It
shows that $h(v)$ is concave when $\rho \ge 0.2$ and $h(v)$ becomes convex for $v$ small when $\rho \le 0.1$.}
\label{fig:h_derivative}
\end{figure}

\begin{figure}
\centering
\begin{minipage}{0.5\textwidth}
  \centering
  \includegraphics[width=3in]{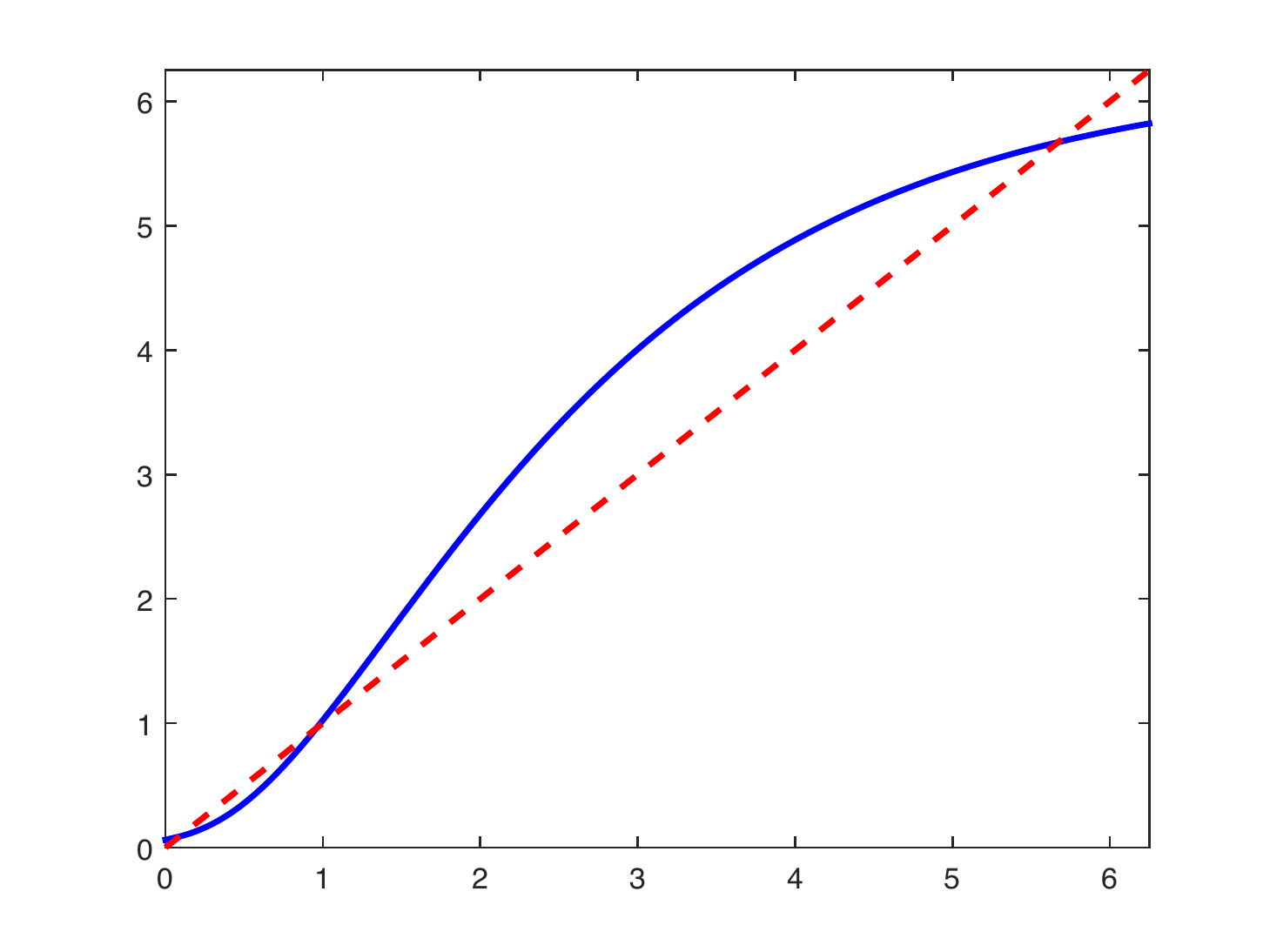}
%  \captionof{figure}{$\rho=0.01$, $\mu=50$, $\nu=0$.}
%  \label{fig:recursion_2}
\end{minipage}%
\begin{minipage}{0.5\textwidth}
  \centering
  \includegraphics[width=3in]{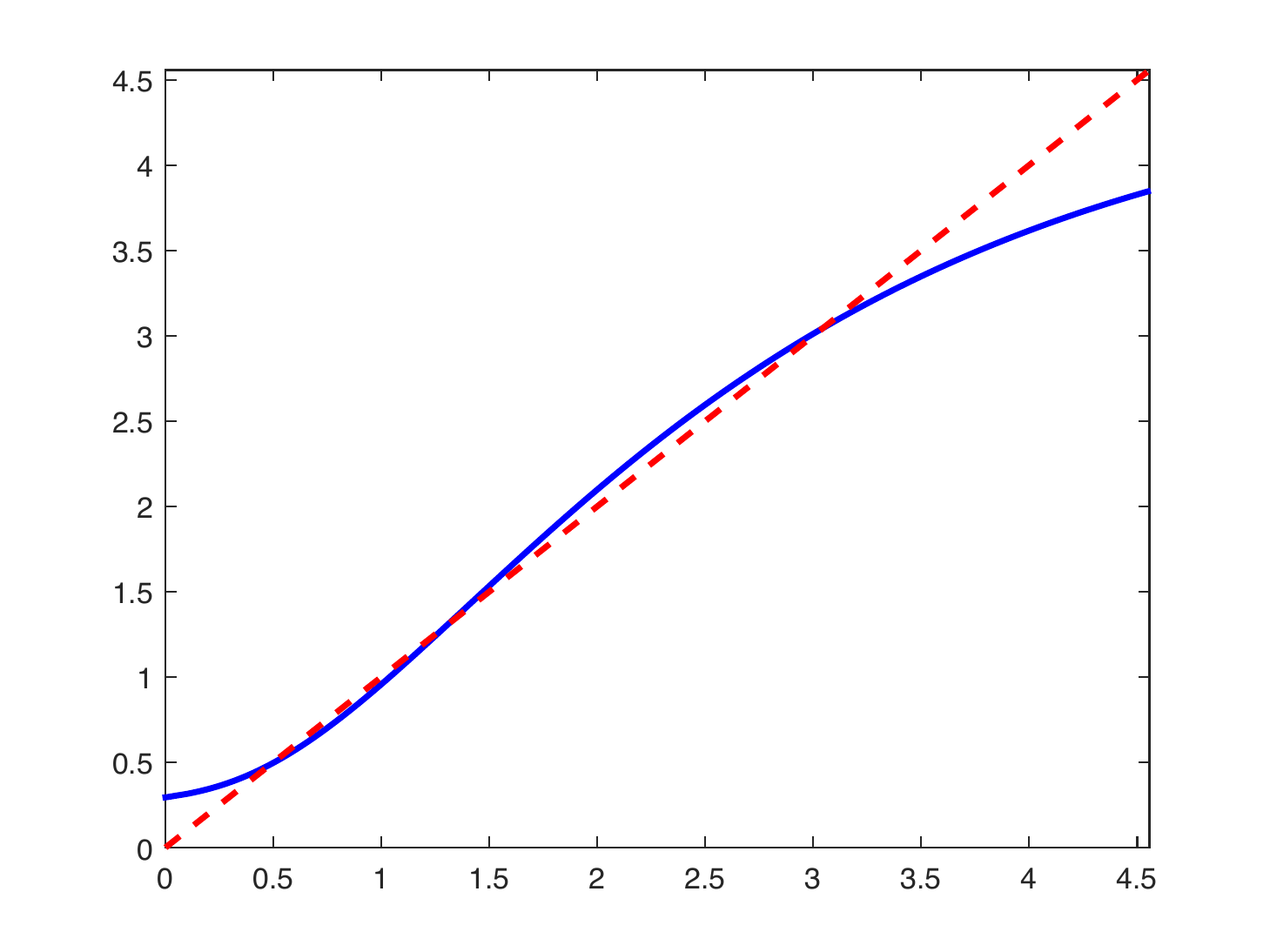}
%  \captionof{figure}{$\rho=0.01$, $\mu=40$, $\nu=1.5$.}
%  \label{fig:recursion_7}
\end{minipage}
\caption{The plot of $\theta+ \lambda h(v)$ (y axis) versus $v$ (x axis) in the case $\rho =0.01$.
Left frame:  $\mu=50$ and $\nu=0$; right frame: $\mu=40$ and $\nu=1.5$. It shows that $v=\theta+ \lambda h(v)$
has three fixed points: The smallest one is $\underline{v}$
corresponding to the performance of local BP; the largest one is $\overline{v}$ corresponding the lower bound on
expected misclassified fraction; the intermediate one is unstable.
}
\label{fig:recursion}
\end{figure}

In the case with $\rho=1/2$ and $\mu \neq \nu$, we show that
 $v= \theta + \lambda h(v)$ has a unique fixed point and thus the local BP is optimal; the key
 idea is to prove that $h(v)$ is concave in this case. Numerical calculations, depicted in Fig.~\ref{fig:h_derivative}, show that $h(v)$ is still concave if $\rho \ge 0.2$,
suggesting that the local BP is still optimal in roughly balanced cluster size cases.
However,
$h(v)$  becomes convex for $v$ small when $\rho \le 0.1$.

%Thus, we have the following conjecture.
%\begin{conjecture}
%There exists a $\delta \in (0,1/2)$ such that if $\rho \in [1/2-\delta, 1/2+\delta]$v= \theta + \lambda h(v)$  has unique fixed point for a
%\end{conjecture}

% \begin{figure}
%\centering
%\begin{subfigure}{0.5\textwidth}
%  \centering
%  \includegraphics[width=.4\linewidth]{figure/recursion_2}
%  \caption{$\rho=0.01$, $\mu=50$, $\nu=0$.}
%  \label{fig:recursion_2}
%\end{subfigure}%
%\begin{subfigure}{0.5\textwidth}
%  \centering
%  \includegraphics[width=.4\linewidth]{figure/recursion_7}
%  \caption{$\rho=0.01$, $\mu=40$, $\nu=1.5$.}
%  \label{fig:sub2}
%\end{subfigure}
%\caption{A figure with two subfigures}
%\label{fig:recursion_7}
%\end{figure}

% But in the case with $\mu=\nu$ and $\rho \neq 1/2$, numerical experiments indicate that there is
%always a unique fixed point.
% \begin{conjecture}
% If $\mu =\nu$, then $v= \theta + \lambda h(v)$ has a unique fixed point for all $\rho \in (0,1/2) \cup (1/2,1)$.
% \end{conjecture}

%Moreover, if $\rho =0.01$,
%numerical experiments, depicted in Fig.~\ref{fig:recursion}, shows that $v= \theta + \lambda h(v)$  may have multiple
%fixed points, suggesting that the local BP may be suboptimal when the cluster sizes
%are very unbalanced. It is an open problem to rigorously justify our numerical findings.

It is intriguing to investigate when $v= \theta + \lambda h(v)$ has a unique fixed point.
If $\rho =0.01$,
numerical experiments, depicted in Fig.~\ref{fig:recursion}, shows that $v= \theta + \lambda h(v)$  may have multiple
fixed points, suggesting that the local BP may be suboptimal.
However, in the case with $\mu=\nu$ and $\rho \neq 1/2$, numerical experiments indicate that there is
always a unique fixed point.
 \begin{conjecture}
 If $\mu =\nu$, then $v= \theta + \lambda h(v)$ has a unique fixed point for all $\rho \in (0,1/2) \cup (1/2,1)$.
 \end{conjecture}
 Notice that in the case with $\mu = \nu$ and $\rho =1/2$, $\theta=0$, $\lambda=\mu^2/4$, and $h(v)=\expect{ \tanh(v+ \sqrt{v} Z)}$.
We have shown in \prettyref{lmm:hmonotone} that $h$ is non-decreasing, concave, and $\lim_{v \to 0} h'(v) =1$. Thus if $|\mu|<2$,
there is a unique fixed point at zero, which is stable; if $|\mu|>2$, there are two fixed points: one is  zero which is unstable and the other is $\overline{v}>0$ which is stable.

 \subsection{Notation and Organization of the Paper}
For any positive integer $n$, let $[n]=\{1, \ldots, n\}$.
For any set $T \subset [n]$, let $|T|$ denote its cardinality and $T^c$ denote its complement.
We use standard big $O$ notations,
e.g., for any sequences $\{a_n\}$ and $\{b_n\}$, $a_n=\Theta(b_n)$
if there is an absolute constant $c>0$ such that $1/c\le a_n/ b_n \le c$.
Let $\Bern(p)$ denote the Bernoulli distribution with mean $p$ and
$\Binom(n,p)$ denote the binomial distribution with $n$ trials and success probability $p$.
All logarithms are natural and we use the convention $0 \log 0=0$.
We say a sequence of events $\calE_n$ holds with high probability if $\prob{\calE_n} \to 1$.

The rest of this paper is organized as follows.
\prettyref{sec:tree} focuses on the inference problems on  the tree model.
The analysis of the belief propagation algorithm on trees
and the proofs of our main theorems are given in \prettyref{sec:densityevolution}.
The  proofs of auxiliary lemmas can be found in \prettyref{app:additional}.

\section{Inference Problems on Galton-Watson Tree Model} \label{sec:tree}

In this section, we first introduce the inference problems
on Galton-Watson trees, and then relate it to the cluster recovery problem under the stochastic block model.

\begin{definition}\label{def:treemodel}
For a vertex $u$, we denote by $(T_u, \tau)$ the following Poisson two-type branching process tree rooted at $u$, where $\tau$ is a $
\{\pm\}$ labeling of the vertices of $T$.
Let $\tau_u =+$ with probability $\rho$ and
$\tau_u=-1$ with probability $\bar \rho$, where $\bar \rho=1- \rho$.
Now recursively for each vertex $i$ in $T_u$, given its label $\tau_i =+$, $i$ will have  $L_i \sim \Pois(\rho a )$ children $j$ with
$\tau_j= + $ and $M_i \sim \Pois(\bar \rho b )$ children $j$ with $\tau_j = -$; given its label $\tau_i=-1$, $i$ will have $L_i \sim \Pois(\rho b)$
children $j$ with $\tau_j=+$ and $M_i \sim \Pois(\bar \rho c)$ children $j$ with $\tau_j=-$.
\end{definition}

For any vertex $i$ in $T_u$, let $T_i^t$ denote the subtree of $T_u$ of depth $t$ rooted at vertex $i$,
and $\partial T_i^t$ denote the set of vertices at the boundary of $T_i^t$.
With a bit abuse of notation, let $\tau_{A}$ denote the vector consisting of labels of vertices in $A$,
where $A$ could be either a set of vertices or a subgraph in $T_u$.
We first consider the problem of estimating the label of root $u$ given the observation of $T_u^t$ and $\tau_{\partial T_u^t}.$
Notice that the labels of vertices in $T_u^{t-1}$ are not observed.

\begin{definition}\label{def:treeestimationexact}
The \emph{detection problem on the tree with exact information at the boundary} is the  problem
of inferring $\tau_u$ from the observation of $ T_u^t$ and $\tau_{\partial T_u^t}.$
The error probability for an estimator
$\hat{\tau}_u(T_u^t, \tau_{\partial T_u^t})$ is defined by
\begin{align*}
p_{T^t} (\hat{\tau}_u) = \rho  \prob{ \hat{\tau}_u=- | \tau_u=+} + \bar \rho  \prob{ \hat{\tau}_u=+ | \tau_u=-}.
\end{align*}
Let $p_{T^t} ^\ast$ denote the minimum error probability among all estimators based on $ T_u^t$ and $\tau_{\partial T_u^t}.$
\end{definition}
The optimal estimator in minimizing $p_{T^t}$, is the maximum a posterior (MAP) estimator,
which can be expressed in terms of log likelihood ratio:
\begin{align*}
\hat{\tau}_{\rm ML} = 2 \times  \indc{\Lambda_u^t  \ge -\varphi } -1,
\end{align*}
where
\begin{align*}
\Lambda^t_i \triangleq \frac{1}{2} \log \frac{ \prob{ T_i^t, \tau_{\partial T_i^t}   | \tau_i=+ }} { \prob{ T_i^t, \tau_{ \partial T_i^t}   |  \tau_i=-} },
\end{align*}
for all $i$ in $T_u$, and $\varphi=\frac{1}{2} \log \frac{\rho}{1-\rho}.$
Thus, the minimum error probability $p_{T^t} ^\ast$  is given by
\begin{align}
p_{T^t} ^\ast=\frac{1}{2} - \frac{1}{2} \expect{ \big|  \prob{ \tau_u=+ | T_u^t, \tau_{\partial T_u^t}   } - \prob{ \tau_u=- | T_u^t,  \tau_{\partial T_u^t } } \big| } .  \label{eq:optimalestimationaccuracytree}
\end{align}

We then consider the problem of estimating $\tau_u$ given observation of $ T_u^t$. Notice that in this case the true labels of vertices in $T_u^{t}$ are not observed.
\begin{definition}\label{def:treeestimation}
The \emph{detection problem on the tree} is the problem
of inferring $\tau_u$ from the observation of $ T_u^t.$ The error probability for an estimator
$\hat{\tau}_u(T_u^t )$ is defined by
\begin{align*}
q_{T^t} (\hat{\tau}_u) = \rho  \prob{ \hat{\tau}_u= - | \tau_u=+} + \bar \rho  \prob{ \hat{\tau}_u= + | \tau_u=-}.
\end{align*}
Let $q_{T^t} ^\ast$ denote the minimum error probability among all estimators based on $ T_u^t.$
\end{definition}
In passing, we remark that the only difference between \prettyref{def:treeestimationexact} and \prettyref{def:treeestimation} is that  the exact labels at the boundary of the tree is revealed to  estimators in the former
and hidden in the latter.
The optimal estimator in minimizing $q_T^t$, is the maximum a posterior  (MAP) estimator, which can be expressed in terms of the log likelihood ratio:
\begin{align*}
\hat \tau_{\rm MAP} =2 \times \indc{\Gamma^t_u \ge - \varphi  } -1 ,
\end{align*}
where
\begin{align*}
\Gamma^t_i  \triangleq  \frac{1}{2} \log \frac{\prob{ \calT^t_i | \tau_u =+ } }{ \prob{ \calT^t_i | \tau_u =- } }.
\end{align*}
for all $i$ in $T_u$, and $\varphi=\frac{1}{2} \log \frac{\rho}{1-\rho}.$ The minimum error probability $q_{T^t} ^\ast$  is given by
\begin{align}
q_{T^t} ^\ast= \frac{1}{2} - \frac{1}{2} \expect{ \big|  \prob{ \tau_u=+ | T_u^t   } - \prob{ \tau_u=- | T_u^t } \big| },  \label{eq:optimalestimationaccuracytreefree}
\end{align}
If $ d_+=d_-$, then the distribution of $T_u^t$ conditional on $\tau_u=+$ is the same
as that conditional on $\tau_u=-$. Thus, $\Gamma^t_u=0$ and the MAP estimator reduces to the trivial estimator, which always guesses the label
to be $+$ if $\rho \ge 1/2$ and $-$ if $\rho <1/2$, and $q_{T^t}^\ast=\min\{\rho, \bar \rho\}.$
If $d_+ \neq d_-$, then $T_u$ becomes statistically correlated with $\tau_u$,
and it is possible to do better than the trivial estimator based on $T_u^t$.

For the tree model, the likelihoods can be computed exactly via a belief propagation algorithm.
The following lemma gives a recursive formula to compute $\Lambda_i^t$ and $\Gamma_i^{t};$
no approximations are needed. Let $\partial i$ denote the set of children of vertex $i$.

 \begin{lemma}\label{lmm:recursion}
 Recall $F(x) = \frac{1}{2} \log \left(  \frac{ \eexp^{2x } \rho a   + \bar \rho b  }{ \eexp^{2x } \rho b +\bar \rho c} \right)$.
For $t \ge 0$,
 \begin{align}
 \Lambda^{t+1}_{i }  & =  \frac{- d_+ + d_- }{2} + \sum_{  j \in \partial i } F( \Lambda_j^t ),  \label{eq:recursionLambda}\\
\Gamma^{t+1}_{i }   & = \frac{- d_+ + d_- }{2}  + \sum_{ j \in \partial i } F( \Gamma_j^t ), \label{eq:recursionGamma}
 \end{align}
 with $\Lambda_i^0= \infty$ if $\tau_i=+$ and $\Lambda_i^0 = -\infty$ if $\tau_i=-$;  $\Gamma^0_{i}  = 0$ for all $i$.
 \end{lemma}

\subsection{Connection between the Graph Problem and Tree Problems} \label{sec:connectiongraphtree}

For the reconstruction problem on graph, recall that $p_{G} (\hat{\sigma}_{\rm BP}^t )$ denote the expected fraction of vertices misclassified by $\hat{\sigma}_{\rm BP}^t$ as per \prettyref{eq:estimationaccuracy}; $p_{G}^\ast$ is the minimum expected misclassified fraction.
For the reconstruction problems on tree, recall that $p_{T^t}^\ast$ is the minimum error probability of estimating $\tau_u$ based on $T_u^t$ and
$\tau_{\partial T_u^t}$ as per \prettyref{eq:optimalestimationaccuracytree};
$q_{T^t}^\ast$ is minimum error probability of estimating $\tau_u$ based on $T_u^t$
as per \prettyref{eq:optimalestimationaccuracytreefree}.
In this section, we show that in the limit $n \to \infty$, $  p_{G} (\hat{\sigma}_{\rm BP}^t ) $ equals to $q_{T^t}^\ast$,
and $p_{G}^\ast$ is bounded by $ p_{T^t}^\ast$ from the below for any $t\ge 1.$
Notice that  $q_{T^t}^\ast$ and $p_{T^t}^\ast$ depend on $n$ only through the parameters $a$, $b$, and $c$.

A key ingredient is to show that $G$ is locally tree-like with high probability in the regime $b=n^{o(1)}$.
Let $G_u^t$ denote the subgraph of $G$ induced by vertices whose distance to  $u$ is at most $t$ and let $\partial G_u^t$ denote the
set of vertices whose distance  from $u$ is precisely $t$. In the following, for ease of notation, we write  $T_u^t$ as $T^t$  and $G_u^t$ as $G^t$ when there is no ambiguity.
With a bit abuse of notation, let $\sigma_{A}$ denote the vector consisting of labels of vertices in $A$,
where $A$ could be either a set of vertices or a subgraph in $G$.
The following lemma proved in \cite{MONESL:15} shows that we can construct a coupling such that
$(G^t, \sigma_{G^t} ) = (T^t, \tau_{T^t} )$ with probability converging to $1$ when $b^{t} =n^{o(1)}.$
\begin{lemma}\label{lmm:couplingtree}
For $t =t(n)$ such that $b^{t} =n^{o(1)}$,
there exists a coupling between $(G, \sigma)$ and $(T, \tau )$
such that $(G^t, \sigma_{G^t}) = (T^t, \tau_{T^t})$ with probability converging to $1$.
\end{lemma}

Suppose that $(G^t, \sigma_{G^t} ) = (T^t, \tau_{T^t} )$ holds, then by comparing
BP iterations \prettyref{eq:mp_commun} and \prettyref{eq:mp_combine_commun} with the recursions of log likelihood ratio $\Gamma^t$ given in \prettyref{eq:recursionGamma},
we find that $R_u^t$ exactly equals to $\Gamma_u^t$, \ie,
the BP algorithm defined in Algorithm \prettyref{alg:MP_commun} exactly computes the log likelihood ratio $\Gamma_u^t$ for the tree model.
Building upon this intuition, the following lemma shows that $  p_{G} (\hat{\sigma}_{\rm BP}^t ) $ equals to $q_{T^t}^\ast$ as $n \to \infty.$

\begin{lemma}\label{lmm:optBPcondition}
For $t =t(n)$ such that $b^{t} =n^{o(1)}$,
\begin{align*}
\lim_{n \to \infty} | p_{G} (\hat{\sigma}_{\rm BP}^t ) - q_{T^t}^\ast | =0.
\end{align*}
\end{lemma}
\begin{proof}
In view of \prettyref{lmm:couplingtree}, we can construct a coupling such that
$(G^t, \sigma_{G^t}) = (T^t, \tau_{T^t} )$ with probability converging to $1$.
On the event $(G^t, \sigma_{G^t} ) = (T^t, \tau_{T^t})$,
we have that  $R_u^t=\Gamma_u^t$. Hence,
\begin{align}
p_{G} (\hat{\sigma}_{\rm BP}^t )  = q_{T^t}^\ast +o(1),  \label{eq:BPaccuracy}
\end{align}
where $o(1)$ term comes from the coupling error.
\end{proof}

The following lemma shows that $ p^\ast_{G}$ is lower bounded by $p_{T^t}^\ast$ as $n \to \infty$.
\begin{lemma}\label{lmm:accuracyupperbound}
For $t =t(n)$ such that $b^{t} =n^{o(1)}$,
\begin{align*}
\limsup_{n \to \infty} \left( p^\ast_{G} - p^\ast_{T^t}  \right) \ge 0.
\end{align*}
\end{lemma}
We pause a while to give some intuition behind the lemma.
To lower bound $p^\ast_G$, it suffices to lower bound the error probability of estimating $\sigma_u$ for a given vertex $u$ based on graph $G$. To this end, we consider
an oracle estimator, which in addition to the graph structure, the exact labels of all vertices at distance exactly $t$ from $u$ are also revealed.
We further show that once the exact labels at distance $t$ are conditioned, $\sigma_u$ becomes asymptotically independent of the labels of all vertices
at distance larger than $t$ from $u$. Hence, effectively the oracle estimator is equivalent to the MAP estimator solely based on the graph structure in $G_u^t$ and the exact
labels at distance $t$. By the coupling lemma, $G_u^t$ is a tree with high probability, and thus the error probability of the oracle estimator asymptotically
equals to $p^\ast_{T^t}.$

\section{Gaussian Density Evolution} \label{sec:densityevolution}
In the previous subsection, we have argued that in the limit $n \to \infty$, $  p_{G} (\hat{\sigma}_{\rm BP}^t ) $ equals to $q_{T^t}^\ast$,
and $p_{G}^\ast$ is bounded by $ p_{T^t}^\ast$ from the below.
In this section, we analyze recursions \prettyref{eq:recursionLambda} and \prettyref{eq:recursionGamma} using density evolution analysis with Gaussian approximations,
and derive simple formulas  for $p^\ast_{T^t}$ and $q^\ast_{T^t}$ in the limit $n \to \infty.$ Afterwards, we give the proof of  \prettyref{thm:optimalaccuracy}.

Notice that $\Gamma_i^t$ is a function of $T_i^t$ alone.
Since the subtrees $\{T_i^t\}_{i \in \partial u}$ conditional on $\tau_u$ are independent and identically distributed,
$\{ \Gamma^t_{i } \}_{i \in \partial u}$ conditional on $\tau_u$ are also independent and identically distributed.
Thus, in view of the recursion \prettyref{eq:recursionGamma}, $\Gamma_u^t$ can be viewed as a sum of i.i.d.\
random variables. When the expected degree of $u$ tends to infinity, due to the central limit theorem,
 we expect that the distribution of $\Gamma_u^t$ conditional on $\tau_u$ is approximately Gaussian.
Moreover, the construction of the subtree $T_i^t$ conditional on $\tau_i$ is the
same as the construction of $T_u^t$ conditional on $\tau_u$. Therefore, for any $i \in \partial u$, the distribution of $\Gamma^t_{i }$ conditional on $\tau_i$ is the same as
the distribution of $\Gamma^t_u$ conditional on $\tau_u$. Similar conclusions hold for $\Lambda_i^t$ as well.

Let $Z_\pm^t $ ($W_\pm^t$) denote a random variable that has the same distribution as $\Gamma_u^t $ ($\Lambda_u^t$) conditional on $\tau_u=\pm$.
The following lemma provides expressions of the mean and variance of $Z_+^t$ and $Z_-^t$. Recall that $\lambda=  \frac{\rho( \mu+\nu)^2}{8} $ and $\theta= \frac{\rho(\mu-\nu)^2 }{8} +  \frac{(1 - 2\rho ) \nu^2}{4}$.
\begin{lemma}
%Let $\lambda=  \rho( \mu+\nu)^2/2 $ and $\theta= \frac{\rho(\mu-\nu)^2 }{2} +  (1 - 2\rho ) \nu^2$.
For all $t \ge 0$,
\begin{align}
\expect{Z_{\pm}^{t+1}} & =  \pm  \theta  \pm \lambda \expect{ \tanh( Z_+^t + \varphi) }  + O( b^{-1/2} ),   \label{eq:meanZ}\\
\var\left( Z_{\pm}^{t+1}\right) & =   \theta+ \lambda  \expect{ \tanh( Z_+^t +  \varphi) }  + O( b^{-1/2} ).\label{eq:varianceZ}
\end{align}
\label{lmm:meanvarianceZ}
\end{lemma}

Recall that  $(v_t: t \ge 0)$ satisfies  $v_0=0$ and
\begin{align*}
v_{t+1} = \theta + \lambda h(v_t)= \theta+ \lambda \expect{ \tanh( v_t+ \sqrt{v_t} Z + \varphi},
\end{align*}
where $Z \sim \calN(0,1)$. Similarly, define $(w_t: t \ge 1)$ by $w_1=\theta+\lambda = \frac{\rho \mu^2+ \bar \rho \nu^2}{4}$ and
\begin{align*}
w_{t+1}= \theta+ \lambda h(w_t)=\theta+ \lambda \expect{ \tanh( w_t+ \sqrt{w_t} Z + \varphi}.
\end{align*}

The following lemma shows that for any fixed $t \ge 0$, $Z_{\pm}^t $ and $W_{\pm}^t$  are approximately Gaussian.
\begin{lemma}\label{lmm:gaussiandensityevolution}
For any $t \ge 0$, as $n \to \infty$,
\begin{align}
\sup_x \bigg|     \prob{  \frac{Z_{\pm}^{t}  \mp v_{t}  }{  \sqrt{ v_t } } \leq x } - \prob{Z \leq x}    \bigg|  =
O(b^{-1/2}).
 \label{eq:gaussianlimit}
\end{align}
Similarly, for any  $t \ge 1$, as $n \to \infty$,
\begin{align}
\sup_x \bigg|     \prob{  \frac{W_{\pm}^{t}  \mp w_{t}  }{  \sqrt{ w_t } } \leq x } - \prob{Z \leq x}    \bigg|  =
O(b^{-1/2}).
 \label{eq:gaussianlimitexact}
\end{align}
\end{lemma}

%\section{State evolution equation}

Before proving \prettyref{thm:optimalaccuracy}, we also need a key lemma, which shows that $h$ is continuous and non-decreasing,
and $h$ is concave if $\varphi=0$.

 \begin{lemma}\label{lmm:hmonotone}
 $h(v)$ is continuous on $[0, \infty)$ and  for $v \in (0, +\infty)$,
 \begin{align}
 h'(v) =  \expect{  \left( 1- \tanh ( v+ \sqrt{v} Z  + \varphi ) \right) \left( 1- \tanh^2( v + \sqrt{v} Z +\varphi)  \right) } \ge 0 \label{eq:hderivative}.
 \end{align}
Furthermore, if $\varphi=0$, then $h'(v) \ge h'(w)$ for $0<v<w<\infty$.
 \end{lemma}

 Finally, we are ready to prove \prettyref{thm:optimalaccuracy} based on \prettyref{lmm:gaussiandensityevolution} and \prettyref{lmm:hmonotone}.
\begin{proof}[Proof of \prettyref{thm:optimalaccuracy}]
 In view of \prettyref{lmm:gaussiandensityevolution},
 \begin{align*}
  \lim_{n \to \infty} \prob{ \Gamma_u^t \ge -\varphi | \tau_u=- }  &=   Q \left(  \frac{v_t - \varphi } { \sqrt{v_t} }\right), \\
  \lim_{n \to \infty} \prob{ \Gamma_u^t  \le -\varphi | \tau_u =+ } & =  Q \left(  \frac{v_t + \varphi } { \sqrt{v_t} }\right).
 \end{align*}
 Hence, it follows from \prettyref{lmm:optBPcondition} that
 \begin{align*}
 \lim_{n \to \infty} p_{G}( \hat{\sigma}_{\rm BP}^t ) =\lim_{n \to \infty} q^\ast_{T^t} = \expect{ Q \left(  \frac{v_t + U } { \sqrt{v_t} }\right) },
 \end{align*}
 where $U= -\varphi$ with probability $1-\rho$ and $U=\varphi$ with probability $\rho$.

 We prove that $v_{t+1} \ge v_{t}$ for $t \ge 0$ by induction.
 Recall that
 \begin{align*}
 v_0 = 0 \le (\rho \mu - \bar \rho \nu )^2/4   = \theta+ \lambda h(v_0) = v_1.
\end{align*}
 Suppose $v_{t+1} \ge v_{t}$ holds; we shall show the claim also holds for $t+1$. In particular, since $h$ is continuous on $[0, \infty)$ and
 differential on $(0, \infty)$, it follows from the mean value theorem that
 \begin{align*}
 v_{t+2} - v_{t+1} = \lambda \left( h (v_{t+1} ) - h(v_{t} ) \right) =  \lambda h' (x) (v_{t+1}-v_{t}),
 \end{align*}
 for some $x \in (v_{t}, v_{t+1} ) $.  \prettyref{lmm:hmonotone} implies that $h'(x) \ge 0$ for $x \in (0, \infty)$,
 it follows that $v_{t+2} \ge v_{t+1}$.  Hence, $v_t$ is non-decreasing in $t$. Next we argue that $v_t \le \underline{v}$
 for all $ t \ge 0$ by induction, where $\underline{v}$ is the smallest fixed point of $v=\theta+ \lambda h(v)$. For the base case, $v_0=0 \le \underline{v}$.
 If $v_t \le \underline{v}$, then by the monotonicity of $h$, $v_{t+1}= \theta+\lambda  h (v_{t} )  \le \theta+ \lambda h ( \underline{v} ) =\underline{v}.$
 Hence, $v_t \le \underline{v}$ and thus $\lim_{t \to \infty} v_t \le \underline{v}$. By the continuity of $h$, $\lim_{t \to \infty} v_t$ is also a fixed point of $v=\theta+ \lambda h(v)$, and consequently $\lim_{t \to \infty} v_t= \underline{v}.$
 Therefore,
\begin{align*}
\lim_{t\to \infty}  \lim_{n \to \infty} p_{G}( \hat{\sigma}_{\rm BP}^t )  =  \lim_{t \to \infty}  \lim_{n \to \infty}  q_{T^t}^\ast
= \expect{ Q \left(  \frac{\underline{v} + U }{ \sqrt{ \underline{v}} }\right) }.
\end{align*}
Next, we prove the claim for $p^\ast_{G}$. In view of \prettyref{lmm:gaussiandensityevolution},
 \begin{align*}
  \lim_{n \to \infty} \prob{ \Lambda_u^t \ge -\varphi | \tau_u=-}   &= Q \left(  \frac{w_t - \varphi } { \sqrt{w_t} }\right),\\
  \lim_{n \to \infty} \prob{ \Lambda_u^t  \le -\varphi | \tau_u =+ } & =  Q \left(  \frac{w_t + \varphi } { \sqrt{w_t} }\right).
 \end{align*}
  Hence, it follows from \prettyref{lmm:accuracyupperbound} that
 \begin{align*}
 \liminf_{n \to \infty} p_{G}^\ast \ge \lim_{n \to \infty} p^\ast_{T^t} = \expect{ Q \left(  \frac{w_t + U } { \sqrt{w_t} }\right) }.
 \end{align*}
 Recall that $w_1=\theta+\lambda \ge w_t$.  By the same argument of proving $v_t$ is non-decreasing, one can show that $w_t$ is non-increasing in $t$.
 Also, by the same argument of proving $v_t$ is upper bounded by $\underline{v}$, one can show that $w_t$ is lower bounded by $\overline{v}$,
 where $\overline{v}$ is the largest fixed point of $v=\theta+ \lambda h(v)$. Thus, $\lim_{t \to \infty} w_t= \overline{v}$
 and $\overline{v} \le w_1=\theta+\lambda= (\rho \mu^2 + \bar \rho \nu^2)/4.$
Therefore,
\begin{align*}
  \liminf_{n \to \infty} p_{G}^\ast  = \lim_{t\to \infty}  \liminf_{n \to \infty} p_{G}^\ast  \ge \lim_{t \to \infty}  \lim_{n \to \infty}  p_{T^t}^\ast
= \expect{ Q \left(  \frac{\overline{v} + U }{ \sqrt{ \overline{v}} }\right) }.
\end{align*}
%Finally,
%notice that
%\begin{align*}
%w_{t+1}- v_{t+1} =\lambda \left( h(w_t) - h(v_t) \right) \le  \lambda (w_t-v_t),
%\end{align*}
%where the last inequality holds because $0\le h'(x) \le 1$.
%Thus, if $\sqrt{\rho} |\mu + \nu| <2\sqrt{2} $, then $\lambda \le 1-\epsilon$ for some $\epsilon>0$,
%and consequently $(w_{t+1}- v_{t+1})\le(1-\epsilon) (w_t-v_t)$.
%Since $w_1-v_1=2(1-\rho)\lambda$, it follows that $\lim_{t\to \infty} (w_t-v_t)=0$ and
%thus $\underline{v}=\overline{v}$.
If $\varphi=0$ and $\mu \neq \nu$,
then $v_1>0$ and \prettyref{lmm:hmonotone} shows that $h'(v) \ge h'(w)$ for all $0<v<w<\infty$.
Since $v_1=\theta+\lambda h(0)>0$ and $\underline{v}=\theta+ \lambda h(\underline{v})$,
it must hold that $\lambda h'(\underline{v})<1$. Thus $\lambda h'(v) <1 $ for all $v \ge  \underline{v}$
and consequently $\theta + \lambda h(v) < v$ for all $v> \underline{v}.$
Hence, $\overline{v}=\underline{v} = v^\ast$, where $v^\ast$ is the unique fixed point of $v= \frac{(\mu-\nu)^2}{16} + \frac{ (\mu+\nu)^2 }{16} \expect{ \tanh(v+ \sqrt{v} Z)}.$
 Therefore,
$$\liminf_{n \to \infty} p_G^\ast \ge   \lim_{t \to \infty}  \lim_{n \to \infty}  p_{T^t}^\ast = \lim_{t \to \infty}  \lim_{n \to \infty}  q_{T^t}^\ast = \lim_{t\to \infty}  \lim_{n \to \infty} p_{G}( \hat{\sigma}_{\rm BP}^t ) = Q(\sqrt{v^\ast}) .$$
Since $p_G^\ast$ is the minimum expected misclassified fraction, it also holds that
$\limsup_{n \to \infty} p_G^\ast \le \lim_{n \to \infty} p_{G}( \hat{\sigma}_{\rm BP}^t )$ for all $t \ge 1$
and consequently
$$\limsup_{n \to \infty} p_G^\ast \le \lim_{t\to \infty}  \lim_{n \to \infty} p_{G}( \hat{\sigma}_{\rm BP}^t ) .$$
Combing the last two displayed equations gives that
$$
\lim_{n \to \infty} p_G^\ast  =  \lim_{t\to \infty}  \lim_{n \to \infty} p_{G}( \hat{\sigma}_{\rm BP}^t ) = Q(\sqrt{v^\ast}) .
$$

 \end{proof}

 \subsection{Degree-Uncorrelated Case}
As remarked in \prettyref{sec:mainresults}, in the case $\rho \mu = \bar \rho \nu$, the vertex degrees
are statistically uncorrelated with the cluster structure, and
no local algorithms is capable of non-trivial detection.
However, it is still possible that local algorithms combined with some efficient global algorithms
achieve  the minimum expected misclassified fraction. In this subsection, we show that it is indeed the case,
if  $\rho=1/2$, $\mu=\nu$ with $|\mu|>2$, and $b=o(\log n)$.
The algorithm as described in Algorithm \prettyref{alg:BPplusCorrelated}  is introduced in \cite{MNS:2013a} and we give the
full description for completeness.
 \begin{algorithm}[htb]
\caption{Local Belief propagation Plus Correlated Recovery}\label{alg:BPplusCorrelated}
\begin{algorithmic}[1]
\STATE Input: $n \in \naturals,$ $a=c, b>0$, $\rho=1/2$, adjacency matrix $A \in \{0,1\}^{n\times n}$, $t \in \naturals$.
%$0<\alpha<1/2$.
\STATE Take $U \subset V$ to be a random subset of size $\lfloor \sqrt{n} \rfloor $.  Let $u_* \in U$ be a random vertex in $U$
with at least $\frac{\log n}{2\log (\log n/b) }$neighbors in $V\backslash U.$
\STATE For $u \in V\backslash U$ do
\begin{enumerate}
\item Run a polynomial-time estimator capable of correlated recovery on the subgraph induced by vertices not in $G_u^{t-1}$ and
$U$, and let $W_u^+$ and $W_u^-$ denote the output of the partition.
\item Relabel $W_u^+$ and $W_u^-$ such that if $a>b$, then $u_*$ has more neighbors in $W_u^+$ than $W_u^-$; otherwise,
$u_*$ has more neighbors in $W_u^-$ than $W_u^+$. Let $\alpha_u$ denote the fraction of vertices misclassified by the partition $(W_u^+,W_u^-)$.
\item For all $i \in \partial G_u^t$ and $j \in \partial G_u^{t-1}$, define $R^{0}_{i \to j} =  \frac{1}{2} \log \frac{1-\alpha_u}{\alpha_u}$ if
$i \in W_u^{+}$, and  $R^{0}_{i \to j} =  - \frac{1}{2} \log \frac{1-\alpha_u}{\alpha_u}$ if
$i \in W_u^{-}$.
\item Run $t-1$ iterations of message passing as in \prettyref{eq:mp_commun}
to compute $R^{t -1 }_{i\to u}$ for all $u$'s neighbors $i$ .
\item Compute  $R_{u}^{t}$ as per \prettyref{eq:mp_combine_commun}, and let $\hat{\sigma}^t_{\rm BP}(u)=+$  if $R_u^t \ge -\varphi$;
otherwise let $\hat{\sigma}^t_{\rm BP} (u)=-$.
\end{enumerate}
\STATE For $u \in U$, let $\hat{\sigma}^t_{\rm BP} (u) $ equal to $+$ or $-$ uniformly at random. Output $\hat{\sigma}_{\rm BP}^t.$
\end{algorithmic}
\end{algorithm}

Notice that Algorithm \prettyref{alg:BPplusCorrelated} runs in time polynomial in $n$.
The algorithm consists of two main steps. First, we apply some global
algorithm to get a correlated clustering when $|\mu|>2$, for example, the algorithm studied in \cite{Mossel13}.
Then, we apply  the local BP algorithm to boost the correlated clustering in the first step
to achieve  the minimum expected misclassified fraction. To ensure the first and second step are independent
of each other, for each vertex $u$, we first withhold the $(t-1)$-local neighborhood of $u$, and then apply the
global detection algorithms on the reduced set of vertices. The clustering on the reduced set of vertices is used
as the initialization to the local belief propagation algorithm running on the withheld $(t-1)$-local neighborhood of $u$.
In this way, the outcome of the global detection algorithm based on the reduced set of vertices is independent of the
edges between the withheld $t$-local neighborhood of $u$ and the reduced set of vertices, as well as the
edges within the withheld set.

There is also a subtle issue to overcome. We run the global detection algorithm once
for each vertex, and the global detection algorithm cannot reliably estimate the sign of the true $\sigma$ due to the symmetry between
$+$ and $-$. Therefore,
different runs of the global detection algorithm may
have different estimates of the sign  of $\sigma$.
We need a way to coordinate different runs of the
global detection algorithms to have the same estimate of the sign of $\sigma$. To this end, a small random subset $U$
is reserved and a vertex of high degree $u_\ast$ in $U$ is served as an anchor. In every runs of the global
detection algorithms, we relabel the partition if necessary, to ensure that $u_\ast$ will always have more neighbors with estimated $+$ labels
than neighbors with estimated
$-$ labels if $a>b$, and the other way around if $a<b$.

Finally, we caution the reader that in addition to the model parameters $a, b$,
 after each run of the global detection algorithm, the algorithm requires knowing $\alpha_u$, which is the fraction of vertices misclassified by
the partition $(W_u^+, W_u^-)$. In the main analysis, we assume the exact value of $\alpha_u$ is known for simplicity.
One can check that only an estimator $\hat{\alpha}_u=\alpha_u+o(1)$ with high probability is needed for \prettyref{thm:symmetry} to hold.
%(See the proof of \prettyref{lmm:treeupperbound_sym} for details).
In \prettyref{app:estimateofalpha}, we give an efficient and data-driven procedure to construct such a consistent  estimate of  $\alpha_u$.

Next, in the limit $n \to \infty$, we give a lower bound on the minimum expected misclassified fraction,
and an upper bound attainable by $\hat{\sigma}^t_{\rm BP}$.
Then we show that the lower and upper bound match with each other
in the double limit, where first $n \to \infty$ and then $t \to \infty$.

Recall that  the fraction of  vertices  misclassified  by $\hat{\sigma}$ is defined up to a global flip of signs of $\hat{\sigma}$ as in \prettyref{eq:defoverlap}.
The following lemma shows that the minimum expected misclassified fraction is still lower bounded by $p^\ast_{T^t}$. Its proof is very similar to the
proof of \prettyref{lmm:accuracyupperbound}. The key new challenge  is that $ \expect{ \calO(\sigma, \hat{\sigma}) }$ does not reduce to the error probability of estimating
$\sigma_u$ for a given vertex $u$ directly.
\begin{lemma}\label{lmm:treelowerbound_sym}
For $t =t (n)$ such that $b^t=n^{o(1)}$,
\begin{align*}
 \limsup_{n \to \infty} \left( \inf_{\hat{ \sigma} } \expect{ \calO(\sigma, \hat{\sigma}) }  - p^\ast_{T^t} \right) \ge 0,
\end{align*}
where $p^\ast_{T^t}$ is defined in \prettyref{eq:optimalestimationaccuracytree} under the tree model with $\rho=1/2$ and $a=c$ defined in
\prettyref{def:treemodel}.
\end{lemma}

In the following, we relate the expected fraction of vertices misclassified by $\hat{\sigma}_{\rm BP}^t$ as defined in  Algorithm \prettyref{alg:BPplusCorrelated}
to an estimation problem on the tree model. In particular, consider the tree model with $\rho=1/2$ and $a=c$ as defined in \prettyref{def:treemodel}.
Fix an $\alpha \in [0,1/2]$. Let $\tilde{\tau}_i =\tau_i $ with probability $1-\alpha$ and $\tilde{\tau}_i = - \tau_i$ with probability for $\alpha$, independently for all $i \in T_u$.
Then $\tilde{\tau}_{\partial T_u^t}$ is a $\alpha$-noisy version of $\tau_{\partial T_u^t}$.
Let  $\tilde{q}_{T^t, \alpha}$ denote the minimum error probability of inferring $\tau_u$
based on $T_u^t$ and $\tilde{\tau}_{\partial T_u^t}$. The optimal estimator achieving $\tilde{q}_{T^t, \alpha}$ is the MAP estimator given by
\begin{align*}
\hat{\tau}_{\rm MAP} = 2 \times \indc{ \tilde{\Gamma}_u^t \ge - \varphi} -1,
\end{align*}
where
\begin{align*}
\tilde{\Gamma}^t_i  \triangleq  \frac{1}{2} \log \frac{\prob{ T^t_i,  \tilde{\tau}_{\partial T_i^t} | \tau_u =+ 1 } }{ \prob{ T^t_i, \tilde{\tau}_{\partial T_i^t}  | \tau_u =-1 } }
\end{align*}
for all $i$ in $T_u$. The minimum error probability   $\tilde{q}_{T^t, \alpha}$
is given by
\begin{align*}
\tilde{q}_{T^t,\alpha} &= \frac{1}{2} \prob{  \tilde{\Gamma}_u^t < -\varphi  | \tau_u=+} + \frac{1}{2} \prob{\tilde{\Gamma}_u^t \ge -\varphi  | \tau_u=-} \\
&= \frac{1}{2} - \frac{1}{2} \expect{  \big|\prob{ \tau_u=+ | T_u^t, \tilde{\tau}_{\partial T_u^t}  | } - \prob{ \tau_u=- | T_u^t, \tilde{\tau}_{\partial T_u^t}  } \big| },
\end{align*}

It follows from the definition that $\tilde{q}_{T^t, \alpha}$ is non-decreasing in $\alpha$. Also, $\tilde{q}_{T^t, \alpha} = p^\ast_{T^t}$ if $\alpha=0$ and $\tilde{q}_{T^t, \alpha} = q^\ast_{T^t}$ if $\alpha=1/2$.
The following lemma shows that the fraction of vertices misclassified by $\hat{\sigma}_{\rm BP}^t$ as defined in  Algorithm
\prettyref{alg:BPplusCorrelated} is asymptotically no larger than $\tilde{q}_{T^t, \alpha}$ for some $\alpha \in [0,1/2)$.
 \begin{lemma} \label{lmm:treeupperbound_sym}
There exists an $\alpha \in [0, 1/2)$ such that for $t =t (n)$ with  $b^t=n^{o(1)}$,
\begin{align*}
 \limsup_{n \to \infty} \left( \expect{ \calO(\sigma, \hat{\sigma}_{\rm BP}^t ) }  -\tilde{q}_{T^t, \alpha} \right) \le 0.
\end{align*}
\end{lemma}

The following lemma gives a characterization of the distribution of $\tilde{\Gamma}_u^t$ based on the density evolution with Gaussian approximations.
\begin{lemma}\label{lmm:gaussian_sym}
Let $\tilde{Z}_{+}^t$ and $\tilde{Z}_{-1}$ denote a random variable that has the same distribution as $\tilde{\Gamma}_u^t$ conditioning on $\tau_u=+$
and $\tau_u=-$, respectively.
For any $t \ge 1$, as $n \to \infty$,
\begin{align}
\sup_x \bigg|     \prob{  \frac{ \tilde{Z}_{\pm}^{t}  \mp u_{t}  }{  \sqrt{ u_t } } \leq x } - \prob{Z \leq x}    \bigg|  = O(b^{-1/2}),
 \end{align}
where $u_1= \frac{(1-2\alpha)^2\mu^2}{4}$ and $u_{t+1}=\frac{\mu^2}{4} \expect{ \tanh(u_{t} + \sqrt{u_t} Z ) }$.
\end{lemma}

We are ready to prove \prettyref{thm:symmetry} by combing \prettyref{lmm:treelowerbound_sym}, \prettyref{lmm:treeupperbound_sym},
and \prettyref{lmm:gaussian_sym}.
 \begin{proof}[Proof of \prettyref{thm:symmetry}]
In view of \prettyref{lmm:gaussian_sym}, for $t \ge 0$,
 \begin{align*}
\lim_{n \to \infty}  \prob{ \tilde{\Gamma}_u^t \ge 0 | \tau_u=-} = \lim_{n \to \infty}  \prob{ \tilde{\Gamma}_u^t \le 0 | \tau_u=+} =Q( \sqrt{u_t} ),
 \end{align*}
It follows from \prettyref{lmm:treeupperbound_sym} that there exists an $\alpha \in [0,1/2)$ such that
 \begin{align*}
   \limsup_{n \to \infty}  \expect{ \calO(\sigma, \hat{\sigma}_{\rm BP}^t ) } \le \lim_{n \to \infty} \tilde{q}_{T^t, \alpha} =  Q(\sqrt{u_t}).
   \end{align*}
 Let $\tilde{h}(v)=\expect{ \tanh(v + \sqrt{v} Z ) }.$
 In view of \prettyref{lmm:hmonotone}, $\tilde{h}$ is non-decreasing and concave in $[0, \infty)$,
 and $\lim_{v \to 0} \tilde{h}'(v) =1$. Notice that $h(0)=0$, and thus $0$ is a fixed point of
 $v=\frac{\mu^2}{4} \tilde{h}(v)$. Moreover, by the mean value theorem,
 for $v>0$, $\tilde{h}(v)=h(0) + \tilde{h}'(\xi) v$ for some $\xi \in (0, v)$. Thus $\frac{\mu^2}{4} \tilde{h}(v) = \frac{\mu^2 }{4} \tilde{h}'(\xi) v$.
 By the assumption that $\mu>2$, and $\lim_{v \to 0} \tilde{h}'(v) =1$, it follows that  there exists a $v_\ast>0$ such that
 $\frac{\mu^2}{4} \tilde{h}(v) > v$ for all $v \in (0, v_\ast)$. Furthermore, $\tilde{h}(v) \le 1$ and hence
$ \frac{\mu^2}{4} \tilde{h}(v) < v$ for all $v$ sufficiently large. Since $\tilde{h}$ is continuous, $v=\frac{\mu^2}{4} \tilde{h}(v)$ must
have nonzero fixed points.  Let $\overline{v}$ denote the smallest nonzero fixed point.
Then $\overline{v}>0$, $ \frac{\mu^2}{4} \tilde{h}(v) > v$ for all $v \in (0, \overline{v})$,
and  $\frac{\mu^2}{4} \tilde{h}' (\overline{v} ) < 1$.
Because $\tilde{h}$ is concave, $ \tilde{h}'(v) \le \tilde{h}'( \overline{v})$ for all $v \ge \overline{v}$.
Thus $ \frac{\mu^2}{4} \tilde{h}(v) > v$ for all $v> \overline{v}$. Therefore, $\overline{v}$ is the unique
nonzero fixed point and also the largest fixed point.
It follows that if $u_1< \overline{v}$, then $\{u_t\}$ is a non-decreasing sequence
upper bounded by $\overline{v}$.   If $u_1>\overline{v}$, then $\{u_t\}$ is a non-increasing sequence
lower bounded by $\overline{v}$. Since $u_1>0$, it follows that $\lim_{t \to \infty} u_t = \overline{v} $.
Hence,
 \begin{align}
\limsup_{n \to \infty}  \inf_{\hat{ \sigma} } \expect{ \calO(\sigma, \hat{\sigma}) }   \le \lim_{t \to \infty} \limsup_{n \to \infty}  \expect{ \calO(\sigma, \hat{\sigma}_{\rm BP}^t ) }  \le \lim_{t \to \infty} \lim_{n \to \infty} \tilde{q}_{T^t, \alpha} = Q( \sqrt{\overline{v} } ).
 \end{align}
 It follows from \prettyref{thm:optimalaccuracy} and \prettyref{lmm:treelowerbound_sym} that
  \begin{align}
\liminf_{n \to \infty}  \inf_{\hat{ \sigma} } \expect{ \calO(\sigma, \hat{\sigma}) }  \ge \lim_{t \to \infty} \lim_{n \to \infty} p^\ast_{T^t} = Q( \sqrt{ \overline{v} } ).
  \end{align}
 The theorem follows by combining the last two displayed equations.
 \end{proof}

\section*{Acknowledgement}
Research supported by NSF grant CCF 1320105, DOD ONR grant N00014-14-1-0823, and grant 328025 from the Simons Foundation. 
J. Xu would like to thank Bruce Hajek and Yihong Wu for numerous discussions on belief propagation algorithms and density evolution analysis.
This work was done in part while J. Xu was visiting the Simons Institute for the Theory of Computing.

 \bibliographystyle{abbrv}
\bibliography{../../graphical_combined,../../other_refs,../../other_refs2}

\appendix
\section{Additional Proofs} \label{app:additional}
\subsection{Proof of \prettyref{lmm:recursion} }
By definition, $\Lambda_i^0=+ \infty$ if $\tau_i=+$ and $\Lambda_i^0 = -\infty$ if $\tau_i=-$, and $\Gamma^0_{i}  = 0$ for all $i$.
We prove the claim for $\Gamma_{i }^t$ with $ t \ge 1$; the claim for $\Lambda_i^t$ with $t \ge 1$ follows similarly.

 A key point is
to use the independent splitting property of the Poisson distribution to give an equivalent description of the numbers of children
of each type for any vertex in the tree.   Instead of separately generating the number of children of each type,  we can first
generate the total number of children and then independently and randomly select the type of each child.
For every vertex $ i$ in $T_u$, let $N_i$  denote the total number of its children.
If $\tau_i=+$ then  $N_i \sim \Pois(d_+)$, and for each child $ j \in \partial i,$ independently of everything else,
$\tau_j=+$ with probability $\rho a/d_+$  and $\tau_j=-$ with probability $\bar \rho b/d_+,$ where $d_+= \rho a +\bar \rho b.$
If   $\tau_i=-$ then  $N_i \sim \Pois( d_-)$, and for each child $j\in \partial i,$ independently of everything else,
$\tau_j=+$ with probability $\rho b / d_-$ and $\tau_j=+$ with probability $\bar \rho c/ d_-,$ where $d_-=\rho b+ \bar \rho c.$
With this view, the observation
of the total number of children $N_i$ of vertex $i$ gives some information, and then the conditionally independent
messages from those children give additional information on $\tau_i$.  Specifically,
\begin{align*}
 \Gamma_i^{t+1}  & =  \frac{1}{2} \log \frac{\prob{T_i^{t+1}  | \tau_i=+ } }{ \prob{  T_i^{t+1}  | \tau_i=-}  }  \overset{(a)}{=}  \frac{1}{2} \log \frac{\prob{N_i | \tau_i=+ }  }{ \prob{N_i | \tau_i= -}  } + \frac{1}{2} \sum_{j \in \partial i} \log  \frac{ \prob{ T^t_j | \tau_i= +} } {  \prob{ T_j^t  | \tau_i=- } }   \\
 & \overset{(b)}{=} \frac{- d_+ + d_- }{2} + \frac{1}{2} N_i  \log \frac{ d_+}{d_-} + \frac{1}{2}  \sum_{ j \in \partial i}  \log \frac{ \sum_{x\in \{\pm \} } \prob{\tau_j=x | \tau_i=+} \prob{ T_j^t | \tau_j =x} }{
 \sum_{x \in \{ \pm  \} } \prob{\tau_j=x | \tau_i=- } \prob{ T_j^t | \tau_j =x} }  \\
 & \overset{(c)}{=}  \frac{- d_+ + d_- }{2}  + \frac{1}{2}  \sum_{j \in \partial i} \log \frac{ \rho a  \prob{ T_j^t | \tau_j =+} + \bar \rho b \prob{ T_j^t | \tau_j =-} }{
 \rho b \prob{ T_i^t | \tau_j =+} + \frac{1}{2}  \bar \rho c \prob{ T_i^t | \tau_j =-}   }\\
 & \overset{(d)}{=}  \frac{- d_+ + d_- }{2}   + \frac{1}{2}  \sum_{j \in \partial i} \log \frac{ \eexp^{2\Gamma^{t}_{j } } \rho a  + \bar \rho b }{  \eexp^{2\Gamma^{t}_{j} } \rho b +\bar \rho c  },
\end{align*}
where $(a)$ holds because $N_i$ and $T^t_j$ for $j \in \partial u$ are independent conditional on $\tau_i$;
$(b)$ follows because $N_i \sim \Pois(d_+)$ if $\tau_i=+$ and $N_i \sim \Pois(d_-)$ if $\tau_i=-$, and $T_j^t$ is independent of $\tau_i$ conditional on $\tau_j$;
$(c)$ follows from the definition of $T_u$ as $\tau_j \sim 2*\Bern(\rho a /d_+) -1 $ (resp. $2*\Bern(\rho b / d_- ) -1 $) conditional on $\tau_i=+$ (resp. $-$);
$(d)$ follows from the definition of $\Gamma^{t}_{j}$.

\subsection{Proof of \prettyref{lmm:accuracyupperbound}}
We will show that as $n \to \infty$, $ p_{G}^\ast $  is bounded by $ p_{T^t}^\ast$ from the below for any $t\ge 1.$
Before that, we need a key lemma which shows that conditional on $(G^t,  \sigma_{\partial G^t })$,
$\sigma_u$ is almost independent of the graph structure outside of $G^t$. The proof is similar to that
of \cite[Proposition 4.2]{MONESL:15} which deals with the special case $\rho=1/2$ and $a=c$.
The key challenge here is that when $\rho \neq 1/2$ or $a \neq c$,
 the overall effect of the non-edges depends on $\sigma$ and some extra care has to be taken (see \prettyref{eq:productAC} for details).

\begin{lemma}\label{lmm:asymptoticalIndependence}
For $t =t(n)$ such that $b^{t} =n^{o(1)},$
there exists a sequence of events $\calE_n$ such that $\prob{\calE_n} \to 1$ as $n \to \infty$, and
%if $\prob{G, \tilde{\sigma}, \sigma_{\partial G^t}, \calE_n} \neq 0,$ then
on event $\calE_n$,
\begin{align}
\prob{ \sigma_ u  =x | G^t,  \sigma_{\partial G^t } } = (1+o(1))  \prob{ \sigma_ u =x | G,  \sigma_{\partial G^t} }, \quad \forall x \in\{\pm\}.
\label{eq:asymInd}
\end{align}
Moreover, on event $\calE_n$,  $(G^t, \sigma_{G^t} ) = (T^t, \tau_{T^t} )$ holds.
\end{lemma}
\begin{proof}
Recall that $G^t$ is the subgraph of $G$ induced by vertices whose distance from $u$ is at most $t$.
Let $A^t$ denote the set of vertices in $G^{t-1}$, $B^t$ denote the set of vertices in $G^{t}$,
and $C^t$ denote the set of vertices in $G$ but not in $G^t$. Then $A^t \cup \partial G^t=B^t$ and $A^t \cup \partial G^t \cup C^t = V.$
Define $s_{A} =\sum_{i \in A^t } \sigma_i$ and
 $s_{C} = \sum_{i \in C^t} \sigma_i$.
Let
\begin{align*}
\calE_n=\{ (\sigma_C, G^t): | s_C | \le n^{0.6},  |B| \le n^{0.1} , (G^t, \sigma_{G^t} ) = (T^t, \tau_{T^t} )\}.
\end{align*}
By the assumption $b^{t}=n^{o(1)}$, it follows that $(G^t, \sigma_{G^t} ) = (T^t, \tau_{T^t})$
and $|B|= n^{o(1)} $ with probability converging to $1$ (see \cite[Proposition 4.2]{MONESL:15} for a  proof).
Note that $s_C = 2 X- |C|$ for some $X \sim \Binom(|C|, \rho )$. Letting $\alpha_n=n^{0.6}$, in view of the Bernstein inequality,
\begin{align*}
 \prob{ \big|s_C - (2 \rho -1) |C|  \big|> \alpha_n } = \prob{  \big| X - \rho |C| \big| > \alpha_n/2 } \le 2 \eexp^{ - \frac{-\alpha_n^2/4}{ |C|/2 + \alpha_n/3 } } =o(1),
 \end{align*}
 where the last equality holds because $ |C| \le n$ and $\alpha_n / \sqrt{n} \to \infty$.
 In conclusion, we have that $\prob{\calE_n} \to 1$ as $n \to \infty$.

To prove that \prettyref{eq:asymInd} holds, it suffices to show that
on event $\calE_n$,
\begin{align}
\prob{ \sigma_ u  =x | G^t,  \sigma_{\partial G^t } } =  (1+o(1))  \prob{ \sigma_ u =x | G^t, \sigma_{\partial G^t }, \sigma_C }, \quad \forall x \in\{\pm\}. \label{eq:asymInd2}
\end{align}
In particular, on event $\calE_n,$
\begin{align*}
\prob{ \sigma_ u  =x | G,   \sigma_{\partial G^t }  }  & = \sum_{\sigma_C } \prob{ \sigma_ u  =x, \sigma_C | G,   \sigma_{\partial G^t }  } \\
& = \sum_{\sigma_C } \prob{ \sigma_C | G,  \ \sigma_{\partial G^t } } \prob{\sigma_u=x | G,  \sigma_{\partial G^t } , \sigma_C } \\
& = \sum_{\sigma_C} \prob{ \sigma_C | G,   \sigma_{\partial G^t } } \prob{\sigma_u=x | G^t,  \sigma_{\partial G^t } , \sigma_C } \\
& =  (1+o(1))  \prob{ \sigma_ u  =x | G^t,   \sigma_{\partial G^t }},
\end{align*}
where the third equality holds, because conditional on $(G^t, \sigma_{\partial G^t }, \sigma_C)$,
$\sigma_u$ is independent of the graph structure outside of $G^t$;
the last equality follows due to \prettyref{eq:asymInd2}.
Hence, we are left to show the desired  \prettyref{eq:asymInd2} holds.

Recall that $G=(V,E)$.
For any two sets $U_1, U_2 \subset V$,
define
\begin{align*}
\Phi_{U_1, U_2} (G, \sigma) = \prod_{ (u,v) \in U_1 \times U_2 } \phi_{uv} (G, \sigma),
\end{align*}
where $(u,v)$ denotes an unordered pair of vertices and
\begin{align}
\phi_{uv}(G,L,\sigma)=\left \{
\begin{array}{rl}
  a/n & \text{if } \sigma_u=\sigma_v =+, (u,v) \in E \\
  c/n & \text{if } \sigma_u=\sigma_v =-, (u,v) \in E \\
 b/n & \text{if } \sigma_u \neq \sigma_v, (u,v) \in E \\
 1-a/n & \text{if } \sigma_u= \sigma_v=+, (u,v) \notin E \\
 1-c/n & \text{if } \sigma_u= \sigma_v=-, (u,v) \notin E\\
 1-b/n & \text{if } \sigma_u \neq \sigma_v, (u,v) \notin E \nonumber
\end{array} \right.
\end{align}
Then the joint distribution of $\sigma$ and $G$ is given by
\begin{align*}
\prob{ \sigma, G, \tilde{\sigma} } = 2^{-n}  \Phi_{B, B} \; \Phi_{C, C} \; \Phi_{\partial G^t, C} \; \Phi_{A, C}.
\end{align*}
Notice that $A$ and $C$ are disconnected.
We claim that on event $\calE_n$, $\Phi_{ A, C} $ only depend on $\sigma_C$ through the $o(1)$ term. In particular, on event $\calE_n$,
\begin{align}
 \Phi_{A, C} (G, \sigma  ) & = \prod_{ (u,v) \in A \times C } \phi_{uv} (G, \sigma )  \nonumber \\
& = \left( 1- \frac{a}{n} \right)^{( |A| + s_A) (|C| + s_C ) /4 } \left( 1- \frac{c}{n} \right)^{( |A| - s_A) (|C| - s_C ) /4 }
\left( 1- \frac{b}{n} \right)^{(|A| |C| - s_A s_C ) /2} \nonumber\\
& = (1+o(1)) \left( 1- \frac{a}{n} \right)^{ \rho ( |A| +s_A )  |C|  /2 } \left( 1- \frac{c}{n} \right)^{ \bar \rho (|A| - s_A) |C|   /2 }   \left( 1- \frac{b}{n} \right)^{ (|A| |C| - s_A (2\rho -1) |C| ) /2 }\nonumber  \\
& \triangleq (1+o(1)) K(\sigma_A, |C| ), \label{eq:productAC}
\end{align}
where the second equality holds because $ u \in A$ and $v \in C$ implies that $(u,v) \notin E$ and thus
$\phi_{uv}$ is either $1-a/n$, $1-c/n$, or $1-b/n$, depending on $\sigma_u$ and $\sigma_v$;
the third equality holds because $ (|A| + |s_A|) |s_C - (2 \rho -1) |C| | \le 2 \alpha_n |B|  =o(n)$;
the last equality holds for some $K(\sigma_A, |C| )$ which only depends on $\sigma_A$ and $|C|$.
 As a consequence,
\begin{align*}
\prob{ \sigma, G,  \calE_n} = (1+o(1)) 2^{-n} \; K(\sigma_A,, |C| )  \;  \Phi_{B, B} \; \Phi_{C, C} \; \Phi_{\partial G^t, C} \; .
\end{align*}
It follows that
\begin{align*}
\prob{ \sigma_ u =x,  G^t, \sigma_{\partial G^t }, \sigma_C, \calE_n  } = (1+o(1)) 2^{-n} \;   \sum_{ \sigma_{A\backslash\{u\} } } K(\sigma_A, |C| )  \;  \Phi_{B, B}
\end{align*}
and thus
\begin{align*}
\prob{ \sigma_ u =x,  G^t, \sigma_{\partial G^t }, \calE_n  } & = \sum_{\sigma_C} \prob{ \sigma_ u =x,  G^t, \sigma_{\partial G^t }, \sigma_C, \calE_n  }\\
& =(1+o(1)) 2^{-n} \;   \sum_{ \sigma_{A\backslash\{u\} } } K(\sigma_A, |C| )  \;  \Phi_{B, B}
\sum_{\sigma_C} \indc{|s_C| \le n^{0.6} }.
\end{align*}
By Bayes' rule,
\begin{align*}
\prob{ \sigma_ u  =x | G^t,  \sigma_{\partial G^t }, \calE_n } & = \frac{\prob{ \sigma_ u =x,  G^t, \sigma_{\partial G^t }, \calE_n  } }{\prob{  G^t, \sigma_{\partial G^t }, \calE_n  }}\\
&= (1+o(1)) \frac{ \sum_{ \sigma_{A\backslash\{u\} } } K(\sigma_A, |C| )  \;  \Phi_{B, B} }{ \sum_{\sigma_{A}}K(\sigma_A, |C| )   \;  \Phi_{B, B} } \\
& = (1+o(1)) \frac{\prob{ \sigma_ u =x,  G^t, \sigma_{\partial G^t }, \sigma_C, \calE_n  }}{\prob{ G^t, \sigma_{\partial G^t }, \sigma_C, \calE_n  }}\\
& = (1+o(1)) \prob{ \sigma_ u  =x | G^t,  \sigma_{\partial G^t }, \sigma_C, \calE_n } .
\end{align*}
Hence, the desired  \prettyref{eq:asymInd2} follows on event $\calE_n$.
\end{proof}

\begin{proof}[Proof of \prettyref{lmm:accuracyupperbound}]
In view of the definition of $p_G^\ast$ given in \prettyref{eq:optimalaccuracygraph},
\begin{align*}
p_{G}^\ast = \frac{1}{2} - \frac{1}{2} \big| \prob{\sigma_u=+  | G } -  \prob{\sigma_u=-  |G } \big| .
\end{align*}
Consider estimating $\sigma_u$ based on $G$.
For any $t \in \naturals$,  suppose a genie reveals the labels of all vertices whose distance from $u$ is precisely $t$, and
let $\hat{\sigma}_{ {\rm Oracle}, t}$ denote the optimal oracle estimator given by
\begin{align*}
\hat{\sigma}_{ {\rm Oracle} ,t} (u) =2 \times \indc{  \prob{\sigma_u =+ | G,  \sigma_{ \partial G^t} } \ge \prob{\sigma_u =- |  G, \sigma_{\partial G^t } } } -1.
\end{align*}
Let $p_{G } (\hat{\sigma}_{ {\rm Oracle},t})  $ denote the error probability  of the oracle estimator, which is given by
\begin{align*}
p_{G } (\hat{\sigma}_{ {\rm Oracle},t} )   =  \frac{1}{2}- \frac{1}{2} \big| \prob{\sigma_u =+ | G, \sigma_{ \partial G^t}} - \prob{\sigma_u =- | G,  \sigma_{\partial G^t } }  \big|
\end{align*}
Since $\hat{\sigma}_{{\rm Oracle}, t}(u)$ is optimal with the extra information $ \sigma_{\partial G^t }$,
it follows that $p_{G} (\hat{\sigma}_{ {\rm Oracle} ,t } ) \le p^\ast_{G}$ for all $t$ and $n$.
\prettyref{lmm:asymptoticalIndependence} implies that there exists a sequence  of events $\calE_n$ such that $\prob{\calE_n \to 1}$ and on event $\calE_n$,
\begin{align*}
\prob{\sigma_u=x |  G, \sigma_{\partial G^t} } = (1+o(1)) \prob{\sigma_u=x |  G^{t}, \sigma_{\partial G^t}  } , \quad \forall x \in \{\pm\},
\end{align*}
and $(G^t, \sigma_{G^t}) = (T^t, \tau_{T^t})$. It follows that
\begin{align*}
p_{G } (\hat{\sigma}_{ {\rm Oracle},t} )
& = \frac{1}{2} - \frac{1}{2}  \expect{ \big| \prob{\sigma_u =+ | G,  \sigma_{ \partial G^t}} - \prob{\sigma_u =- | G,  \sigma_{\partial G^t } }  \big|  \indc{\calE_n} }  + o(1) \\
& = \frac{1}{2} - \frac{1}{2}   \expect{ \big| \prob{\sigma_u =+ | G^{t}, \sigma_{ \partial G^t} } - \prob{\sigma_u =- | G^{t},  \sigma_{\partial G^t } }  \big|   \indc{\calE_n} }  + o(1) \\
& =  \frac{1}{2}  - \frac{1}{2}   \expect{ \big|  \prob{ \tau_u=+  |  T^{t}, \tau_{\partial T^t} }  -  \prob{ \tau_u=-  |  T^{t},   \tau_{\partial T^t} } \big|   \indc{\calE_n} } + o(1) \\
& =  \frac{1}{2} -\frac{1}{2}   \expect{ \big|  \prob{ \tau_u=+  | T^{t},     \tau_{\partial T^t} }  -  \prob{ \tau_u=-  |T^{t},   \tau_{\partial T^t}  } \big| }   + o(1) \\
&= p_{T^t}^\ast  + o(1).
\end{align*}
Hence,
\begin{align*}
\limsup_{n \to \infty} \left( p^\ast_{G} - p_{T^t}^\ast \right) \geq  \limsup_{n \to \infty} \left( p_{G } (\hat{\sigma}_{ {\rm Oracle} ,t } ) - p_{T^t}^\ast \right) =0.
\end{align*}
\end{proof}

\subsection{Proof of \prettyref{lmm:meanvarianceZ}}
We first prove the claims for $Z_-^{t+1}$.
By the definition of $\Gamma_u^t$ and the change of measure, we have
 \begin{align*}
 \expect{ g (\Gamma_u^t) | \tau_u= -  } = \expect{ g( \Gamma_u^t ) \eexp^{ - 2 \Gamma_u^t} | \tau_u=+},
 \end{align*}
 where $g$ is any measurable function such that the expectations above are well-defined. It follows that
 \begin{align}
 \expect{ g(Z_-^t) } = \expect{ g(Z_+^t) \eexp^{-2 Z_+^t} }. \label{eq:nishimori}
 \end{align}
Define $\psi(x) \triangleq  \log (1+x)- x + x^2/2$.
It follows from the Taylor expansion that
$ |\psi(x) | \le |x|^3$. Then
\begin{align*}
 F(x) & =  \frac{1}{2} \log \left(  \frac{ \eexp^{2x} \rho a + \bar \rho b      }{ \eexp^{2x} \rho b  + \bar \rho c }  \right)
  = \frac{1}{2} \log \frac{b}{c} +  \frac{1}{2}\log \left( \frac{ \eexp^{2x} (\rho a)/(\bar \rho b)    +1   }{ \eexp^{2x} (\rho b)/(\bar \rho c)  + 1 } \right) \\
 & = \frac{1}{2}\log \frac{b}{c}   + \frac{1}{2} \log \left( 1 + \frac{ \eexp^{4\beta}  -1 }{1+ \eexp^{-2x} (\bar \rho c)/(\rho b)  } \right) \\
 & = \frac{1}{2} \log \frac{b}{c} +  \frac{\left( \eexp^{4\beta}  -1\right)}{2}  f(x) - \frac{ \left( \eexp^{4\beta}  -1 \right)^2 }{4} f^2(x) + \frac{1}{2} \psi\left( (\eexp^{4\beta}  -1) f(x)\right),
\end{align*}
where $ \beta = \frac{1}{2} \log \frac{\sqrt{ac} }{b}$ and $f(x)= \frac{1}{1+ \eexp^{-2x} (\bar \rho c)(\rho b) }.$ Since $ |\psi(x) | \le |x|^3$ and $|f(x)| \le 1$, it follows that
\begin{align}
 F(x)
  = \frac{1}{2} \log \frac{b}{c}  +  \frac{\left( \eexp^{4\beta}  -1\right)}{2}  f(x)  - \frac{ \left( \eexp^{4\beta}  -1 \right)^2 }{4 } f^2(x) + O \left( |\eexp^{4\beta}  -1|^3 \right), \label{eq:ApproxF}
\end{align}
Therefore, in view of \prettyref{eq:recursionGamma},
\begin{align*}
\Gamma_u^{t+1} & = \frac{- d_+ + d_- }{2}  + \sum_{ \ell \in \partial u } F( \Gamma_\ell^t ) \\
& =\frac{-d_+ + d_- }{2}  + \frac{1}{2} \sum_{\ell \in \partial u }  \left[ \log \frac{b}{c}  + \left( \eexp^{4\beta}  -1 \right) f(  \Gamma_\ell^{t } )
-   \frac{ \left( \eexp^{4\beta}  -1 \right)^2 }{2} f^2( \Gamma_\ell^{t } ) +  O \left( |\eexp^{4\beta}  -1|^3 \right)  \right].
\end{align*}
By conditioning the label of vertex $u$ is $-$, it follows that
\begin{align*}
\expect{Z_-^{t+1}} &=\frac{- d_+ + d_-}{2} + \frac{1}{2} \log (b/c) d_-  +  \frac{\left( \eexp^{4\beta}  -1\right)}{2}  \left( \rho b  \expect{  f (Z_+^t  )    } +  \bar \rho c  \expect{ f (Z_-^t  ) }
\right) \\
& - \frac{\left( \eexp^{4\beta}  -1 \right)^2  }{4} \left(  \rho b \expect{ f^2 (Z_+^t  ) } +  \bar \rho c \expect{ f^2 (Z_-^t  )  } \right) +  O \left(b  |\eexp^{4\beta}  -1|^3 \right).
\end{align*}
In view of \prettyref{eq:nishimori}, we have that
\begin{align}
\rho b  \expect{  f (Z_+^t  )    } +  \bar \rho c  \expect{ f (Z_-^t  ) } & =  \rho b \expect{f (Z_+^t  ) ( 1+ \eexp^{-2 Z_+^t }  (\bar \rho c)/(\rho b)  )}   =  \rho b,  \label{eq:nishimori1} \\
\rho b \expect{ f^2 (Z_+^t  ) } +  \bar \rho c \expect{ f^2 (Z_-^t  ))}  & = \rho b \expect{f^2 (Z_+^t  ) ( 1+ \eexp^{- 2Z_+^t }  (\bar \rho c)/(\rho b)  )}  = \rho b \expect{f (Z_+^t  ) } \label{eq:nishimori2}.
\end{align}
Hence,
\begin{align*}
\expect{Z_-^{t+1}} = \frac{- d_+ + d_-}{2}   +\frac{1}{2}  \log (b/c) d_-   +    \frac{\left( \eexp^{4\beta}  -1\right)}{2}  \rho b
- \frac{ \left( \eexp^{4\beta}  -1 \right)^2  \rho b } {4}  \expect{f (Z_+^t )}  +  O \left( b |\eexp^{4\beta}  -1|^3 \right) .
\end{align*}
Notice that
\begin{align}
\log \frac{b}{c}= - \log \left(1+ \frac{c-b}{b} \right) = \frac{b-c}{b} + \frac{(b-c)^2}{2b^2} + O \left(  \frac{|b-c|^3}{ b^3 } \right). \label{eq:Approxlogbc}
\end{align}
As a consequence,
\begin{align*}
& - d_+ + d_- + \log (b/c) d_-   +    \left( \eexp^{4\beta}  -1\right)  \rho b  \\
&= - \rho a - \bar \rho b + b + \bar \rho (c-b) + \log (b/c)( b + \bar \rho (c-b) ) + \rho \frac{ac-b^2}{b}\\
&= - \rho a - \bar \rho b + b+ \bar \rho (c-b) + (b-c) - \frac{\bar \rho (b-c)^2}{b} + \frac{ (b-c)^2}{2b} + \rho \frac{ac-b^2}{b}+ O \left(  \frac{|b-c|^3}{ b^2 } \right) \\
&=  \rho ( -a+ b -c + ac/b  ) + (1/2-\bar \rho ) (b-c)^2/b +O \left(  \frac{|b-c|^3}{ b^2 } \right) \\
&= \rho (a-b)(c-b) /b + (\rho -1/2) (b-c)^2/b +O \left(  \frac{|b-c|^3}{ b^2 } \right)  \\
& = \rho \mu \nu + (\rho -1/2) \nu^2  + O (b^{-1/2}),
\end{align*}
where the last equality holds due to $(a-b)/\sqrt{b}= \mu$ and $(c-b)/\sqrt{b}= \nu$ for  fixed constants $\mu$ and $\nu$.
Moreover,
\begin{align*}
 \left( \eexp^{4\beta}  -1 \right)^2  \rho b    = \frac{ \rho (ac-b^2)^2 }{b^3} = \rho  \left(  \frac{a-b}{\sqrt{b}} + \frac{c-b}{\sqrt{b}} + \frac{(a-b)(c-b)}{b \sqrt{b} } \right)^2
=\rho ( \mu+ \nu)^2) + O (b^{-1/2} ),
\end{align*}
and $ b |\eexp^{4\beta}  -1|^3 =  O ( b^{-1/2}).$
Assembling the last four displayed equations gives that
\begin{align*}
\expect{Z_-^{t+1}} = \frac{1}{2}\rho \mu \nu + \frac{(2\rho -1) \nu^2}{4} - \frac{ \rho ( \mu+ \nu )^2  }{4}  \expect{ f(Z_+^t)  }  +O (b^{-1/2} ).
\end{align*}
Finally, recall that $\varphi=\frac{1}{2} \log \frac{\rho}{1-\rho}$ and thus
\begin{align}
\bigg| f(x) -\frac{1}{1+ \eexp^{- 2(x+ \varphi) } }  \bigg| & = \frac{\eexp^{-x}  (\bar \rho /\rho) | 1- c/b| }{ ( 1+ \eexp^{-x } (\bar \rho c)/(\rho b)  ) (1+ \eexp^{-x } (\bar \rho/ \rho) ) }
\le  | 1- c/b| = O( b^{-1/2} ). \label{eq:Approxf}
\end{align}
It follows that
\begin{align*}
\expect{Z_-^{t+1}} &=   \frac{1}{2}\rho \mu \nu + \frac{(2\rho -1) \nu^2}{4} - \frac{ \rho ( \mu+ \nu )^2  }{4} \expect{  \frac{1}{1+ \eexp^{-2( Z_+^t+\varphi) }  } }  +O (b^{-1/2} ) \\
&= - \frac{\rho(\mu-\nu)^2 }{8} +  \frac{(2\rho -1) \nu^2}{4}  - \frac{ \rho ( \mu+ \nu )^2  }{8}  \expect{ \tanh( Z_+^t + \varphi) }  +O (b^{-1/2} ),
\end{align*}
where in the last equality we used the fact that $\frac{1}{1+\eexp^{-x}} = \frac{1}{2} (\tanh(x) +1).$
Recall that $\lambda=  \frac{\rho( \mu+\nu)^2}{8} $ and $\theta= \frac{\rho(\mu-\nu)^2 }{8} +  \frac{(1 - 2\rho ) \nu^2}{4}$.
Therefore, we get the desired equality:
\begin{align*}
\expect{Z_-^{t+1}} = - \theta -\lambda  \expect{ \tanh( Z_+^t + \varphi) }  +O (b^{-1/2} ).
\end{align*}

Next we calculate $\var(Z_-^{t+1})$. For $Y= \sum_{i=1}^L X_i$, where $L$ is Poisson distributed, and $\{ X_i \}$ are i.i.d.\ with finite second moments,
one can check that $\var(Y) = \expect{L} \expect{X_1^2}$.
In view of \prettyref{eq:recursionGamma},
\begin{align*}
\var (Z_-^{t+1}) = \rho b \expect{ F^2 (Z_+^t ) } + \bar \rho c \expect{ F^2 ( Z_-^t ) },
\end{align*}
In view of \prettyref{eq:ApproxF} and the fact that $\eexp^{4\beta} -1 =o(1)$, we have that
\begin{align*}
F^2(x) = \frac{1}{4} \log^2 \left(  \frac{b}{c} \right)  + \frac{1}{2} \log \frac{b}{c} \left(\eexp^{4\beta} -1 \right)  f(x) +  \frac{1}{4} \left( \eexp^{4\beta}  -1 \right)^2 (1 -  \log (b/c) ) f^2(x)  + O \left( |\eexp^{4\beta}  -1|^3 \right),
\end{align*}
Thus,
\begin{align*}
\var (Z_-^{t+1}) & = \frac{1}{4} \log^2 \left(  \frac{b}{c} \right)  d_-   + \frac{1}{2} \log \frac{b}{c} \left(\eexp^{4\beta} -1 \right)  \left[ \rho b \expect{ f (Z_+^t ) } + \bar \rho c \expect{ f (Z_-^t  ) }   \right] \\
& +   \frac{1}{4} \left( \eexp^{4\beta}  -1 \right)^2 (1 -  \log (b/c) ) \left[ \rho b \expect{ f^2 (Z_+^t ) } + \bar \rho c \expect{ f^2 (Z_-^t  ) }   \right] + O \left(b |\eexp^{4\beta}  -1|^3 \right)
\end{align*}
Applying \prettyref{eq:nishimori1} and \prettyref{eq:nishimori2}, we get that
\begin{align*}
\var (Z_-^{t+1}) = & \frac{1}{4} \log^2 \left(  \frac{b}{c} \right)  d_-   + \frac{ \rho b  \left(\eexp^{4\beta} -1 \right) }{2}   \log \frac{b}{c}
 +  \frac{ \rho b \left( \eexp^{4\beta}  -1 \right)^2 }{4}  \left(1 -  \log \frac{b}{c} \right)  \expect{ f (Z_+^t ) } \\
 &+ O \left(b |\eexp^{4\beta}  -1|^3 \right).
\end{align*}
In view of \prettyref{eq:Approxlogbc}, we have that
\begin{align*}
\log^2 (b/c)  d_-   +2  \rho b \left(\eexp^{4\beta} -1 \right) \log (b/c) & = \log^2(b/c) (b + \bar \rho (c-b) ) + 2 \rho  \frac{ac-b^2}{b}  \log (b/c)   \\
&= \frac{ (b-c)^2 }{b} + 2 \rho (b-c) \frac{ac-b^2}{b^2} + O\left( \frac{|b-c|^3}{b^2}\right) \\
&= \nu^2 -2 \rho \nu (\mu + \nu) + O(b^{-1/2}),
\end{align*}
and that
\begin{align*}
 \rho b \left( \eexp^{4\beta}  -1 \right)^2 (1 -  \log (b/c) )  = \frac{ \rho (ac-b^2)^2 }{b^3 }  (1 -  \log (b/c) )   = \rho (\mu + \nu)^2 +o(1).
\end{align*}
Moreover, we have shown that $b |\eexp^{4\beta}  -1|^3=O(b^{-1/2})$.
Assembling the last three displayed equations give that
\begin{align*}
\var (Z_-^{t+1}) =   \frac{\nu^2}{4}  - \frac{\rho \nu (\mu + \nu)}{2}   + \frac{\rho (\mu + \nu)^2}{4} \expect{ f(Z_+^t) }+ O(b^{-1/2}).
\end{align*}
Finally, in view of \prettyref{eq:Approxf}, we get that
\begin{align*}
\var (Z_-^{t+1}) &= \frac{\nu^2}{4}  - \frac{\rho \nu (\mu + \nu)}{2}   + \frac{\rho (\mu + \nu)^2}{4} \expect{ \frac{1}{ 1+ \eexp^{-2 (Z_+^t+ \varphi) }  }  }+ O(b^{-1/2})\\
 & = \frac{ \rho (\mu-\nu)^2}{8} + \frac{(1-2\rho) \nu^2}{4} +  \frac{\rho (\mu + \nu)^2}{8} \expect{  \tanh(Z_+^t + \varphi ) }+ O(b^{-1/2}) \\
 & = \theta + \lambda  \expect{  \tanh(Z_+^t + \varphi ) }+ O(b^{-1/2}).
 \end{align*}

The claims for $Z_+^{t+1}$ can be proved similarly as above. We provide another proof by exploiting the symmetry.
In particular, note that our tree model is parameterized by $(\rho, a, b,c)$ with labels $+$ and $-$.
Consider another parametrization $(\rho',a',b',c')$ with labels $+'$ and $-'$, where $\rho'=\bar \rho$, $a'=c$, $b'=b$, $c'=a$, $+'=-$, $-'=+$.
Let $\tilde{Z}_{+'}^{t}$ and $\tilde{Z}_{-'}^{t}$ denote the random variables corresponding to $Z_{+}^t$ and $Z_{-}^t$, respectively.
Then, one can check that $\tilde{Z}_{+'}^t$ has the same distribution as $-Z_{-}^t$ and $\tilde{Z}_{-'}^t$ has the same distribution as $-Z_+^t$.
We have shown that
\begin{align*}
\expect{\tilde{Z}_{-'}^{t+1}} =  - \frac{\rho'(\mu'-\nu')^2 }{8} +  \frac{(2\rho' -1) (\nu')^2}{4}  - \frac{ \rho' ( \mu'+ \nu' )^2  }{8}  \expect{ \tanh( \tilde{Z}_{+'}^t + \varphi') }  +O (b^{-1/2} ),
\end{align*}
where $\mu'= \frac{a'-b'}{b'}=\frac{c-b}{b}=\nu$ and similarly $\nu'=\mu$, and $\varphi'=-\varphi =\frac{1}{2} \log \frac{1-\rho}{\rho}$.
It follows that
\begin{align}
\expect{Z_{+}^{t+1} }=  \frac{\bar \rho(\mu-\nu)^2 }{8} -  \frac{(2 \bar \rho -1) \mu^2}{4}  - \frac{ \bar \rho ( \mu+ \nu )^2  }{8}  \expect{ \tanh( Z_{-}^t + \varphi) }  +O (b^{-1/2} ),
\label{eq:Zplusmeansym}
\end{align}
Applying $g(x)=\tanh(x-\varphi)$ into \prettyref{eq:nishimori}, we get that
\begin{align*}
  \expect{ \tanh( Z_{+}^t + \varphi) } + \eexp^{ -2 \varphi} \expect{ \tanh( Z_{-}^t + \varphi) }
  & =  \expect{ \tanh( Z_{+}^t + \varphi)\left( 1+ \eexp^{-2 (Z_{+}^t +\varphi) } \right) } \\
  & = \expect{ 1 - \eexp^{-2 (Z_{+}^t + \varphi)  } } = 1- \eexp^{-2\varphi},
\end{align*}
where the last equality by the change of measure: $\expect{ \eexp^{-Z_+^t} } = 1$. Hence,
\begin{align*}
 \expect{ \tanh( Z_{-}^t + \varphi) } = \frac{1- \eexp^{-2\varphi} - \expect{ \tanh( Z_{+}^t + \varphi) } }{\eexp^{-2 \varphi} } = \frac{\rho-\bar \rho - \rho \expect{ \tanh( Z_{+}^t + \varphi) } }{1-\rho }.
\end{align*}
It follows from \prettyref{eq:Zplusmeansym} that
\begin{align*}
\expect{Z_{+}^{t+1} }
&=  \frac{\bar \rho(\mu-\nu)^2 }{8} -  \frac{(2 \bar \rho -1) \mu^2}{4}  -  \frac{  ( \mu+ \nu )^2 (\rho -\bar \rho)  }{8}
+\frac{ \rho ( \mu+ \nu )^2  }{8}  \expect{ \tanh( Z_{+}^t + \varphi) }  +O (b^{-1/2} )\\
& = \frac{\rho(\mu-\nu)^2 }{8} +  \frac{(1-2\rho ) \nu^2}{4}  + \frac{ \rho ( \mu+ \nu )^2  }{8}  \expect{ \tanh( Z_+^t + \varphi) }  +O (b^{-1/2} )\\
& = \theta+ \lambda \expect{ \tanh( Z_+^t + \varphi) }  +O (b^{-1/2} ).
\end{align*}
Finally, note that
\begin{align*}
\var(Z_+^{t+1}) = \var( \tilde{Z}_{-'}^{t+1}) = - \expect{\tilde{Z}_{-'}^{t+1}} +O(b^{-1/2})  =\expect{Z_{+}^{t+1} } + O(b^{-1/2}).
\end{align*}
Combing the last two displayed equations completes the proof.

\subsection{Proof of \prettyref{lmm:gaussiandensityevolution}}

The following lemma is useful for proving the distributions of $Z_+^t$ and $Z_-^t$ are approximately Gaussian.
\begin{lemma}(Analog of Berry-Esseen inequality for Poisson sums \cite[Theorem 3]{korolev2012improvement}.)\label{lmm:Poisson_BE}
Let  $S_{d}=X_1 + \cdots + X_{N_d},$   where
$X_i: i\geq 1$  are independent, identically distributed random variables with finite second moment, and for some $d > 0,$ $N_{d}$ is a $\Pois(d)$ random variable independent
of  $(X_i: i\geq 1).$   Then
$$
\sup_x \bigg|     \prob{  \frac{S_d - \expect{S_d} }{  \sqrt{\var(S_d)  }}\leq x} - \prob{Z \leq x}  \bigg|  \leq  \frac{C_{BE} \expect{|X_1|^3}}{\sqrt{d ( \expect{X_1^2} )^3}},
$$
where $\expect{S_d}=d \expect{X_1}$, $\var(S_d)=d \expect{X_1^2}$,  and $C_{BE}=0.3041.$
\end{lemma}

\begin{proof}[Proof of \prettyref{lmm:gaussiandensityevolution}]
 We prove the lemma by induction over $t$.  We first consider the base case. For $Z^t$, the base case $t=0$
trivially holds, because  $\Gamma_u^0 \equiv 0$ and $v_0=0$. For $W^t$, we need to check the base case $t=1$.
Recall that $\Lambda^0_\ell=\infty$ if $\tau_\ell=+$ and  $\Lambda^0_\ell= -\infty$ if $\tau_\ell=-$.
Notice that $F(\infty)=\frac{1}{2}\log(a/b)$ and $F(-\infty)=\frac{1}{2} \log(b/c)$. Hence,
\begin{align}
\Lambda_u^1 = \frac{-d_+ + d_-}{2} + \sum_{i=1}^{N_{d} } X_i, \label{eq:recursionLambdaone}
\end{align}
where conditional on $\tau_u=\pm$, $N_d \sim \Pois(d_\pm )$ is independent of $\{X_i\}$; $\{X_i\}$ are i.i.d.\ such that
conditional on $\tau_u=+$, $ X_i=\frac{1}{2} \log(a/b)$ with probability $(\rho a)/d_+$ and $X_i=\frac{1}{2}  \log(b/c) $ with probability $(\bar \rho b)/d_+$;
conditional on $\tau_u=-$, $ X_i=\frac{1}{2}  \log(a/b)$ with probability $(\rho b)/d_-$ and $X_i=\frac{1}{2}  \log(b/c) $ with probability $(\bar \rho c)/d_-$.
Taylor expansion yields that
\begin{align*}
\log (a/b) &= \log \left(1+ \frac{a-b}{b} \right) =\frac{a-b}{b} - \frac{\mu^2}{2b} + O \left( b^{-3/2} \right)\\
\log (b/c) &=  - \log \left(1+ \frac{c-b}{b} \right) = \frac{b-c}{b} + \frac{\nu^2}{2b} + O \left( b^{-3/2} \right),
\end{align*}
Since $F$ is monotone,
\begin{align*}
\expect{X_1^2} & \ge  \min \{ |\log (a/b)|^2, | \log (b/c)|^2 \} = \Omega \left( \min \left\{  \frac{(a-b)^2}{b^2}, \frac{(c-b)^2}{b^2} \right\} \right) =\Omega( b^{-1}) \\
\expect{|X_1|^3} & \le \max\{ |\log (a/b)|^3, | \log (b/c)|^3 \} =O \left(  \frac{|a-b|^3+ |b-c|^3}{b^3}  \right) = O( b^{-3/2}).
\end{align*}
Thus, in view of \prettyref{lmm:Poisson_BE},
we get that
$$
\sup_x \bigg|     \prob{  \frac{ W_{\pm}^1 - \expect{W_{\pm}^1 } }{  \sqrt{ \var \left( W_{\pm}^1 \right)  }  } \leq x} -  \prob{Z \leq x}   \bigg|  \leq  O (b^{-1/2} ).
$$
By conditioning  the label of $u$ is $-$, it follows from \prettyref{eq:recursionLambdaone} that
\begin{align*}
\expect{W^1_-} &= \frac{1}{2} \left[ -d_+ + d_- +  \log(a/b) \rho b +  \log(b/c)\bar \rho c \right] = - \frac{\rho \mu^2 + \bar \rho \nu^2}{4}+ O(b^{-1/2}) = -w_1 +  O(b^{-1/2})  \\
\var \left( W^1_- \right) & = \frac{1}{4}\log^2(a/b) \rho b + \frac{1}{4}\log^2(b/c) \bar \rho b =  \frac{\rho \mu^2 + \bar \rho \nu^2}{4}+ O(b^{-1/2}) =w_1 + O(b^{-1/2}),
\end{align*}
where we used the fact that $w_1=\theta+\lambda = \frac{\rho \mu^2 + \bar \rho \nu^2}{4}$ by definition.
Similarly, by conditioning  the label of $u$ is $+$, it follows that
\begin{align*}
\expect{W^1_+} &= \frac{1}{2} \left[ -d_+ + d_- +  \log(a/b) \rho a +  \log(b/c)\bar \rho b \right] = \frac{\rho \mu^2 + \bar \rho \nu^2}{4}+ O(b^{-1/2}) =w_1 + O(b^{-1/2}) \\
\var \left( W^1_+ \right) &= \frac{1}{4} \log^2(a/b) \rho a + \frac{1}{4}\log^2(b/c)\bar \rho b =\frac{\rho \mu^2 + \bar \rho \nu^2}{4}+ O(b^{-1/2})=w_1 + O(b^{-1/2}).
\end{align*}
Hence, we get the desired equality:
$$
\sup_x \bigg|     \prob{  \frac{ W_{\pm}^1 \mp w_1 }{  \sqrt{ w_1 }  } \leq x} -  \prob{Z \leq x}   \bigg|  \leq  O (b^{-1/2} ).
$$

In view of \prettyref{eq:recursionLambda} and \prettyref{eq:recursionGamma}, $\Lambda^t$ and $\Gamma^t$ satisfy the same recursion. Moreover, by definition,
$v_t$ and $w_t$ also satisfy the same recursion. Thus, to finish the proof of the lemma, it suffices to show that:
 suppose \prettyref{eq:gaussianlimit}
holds for  $t$, then it also holds for $t+1.$
We prove the claim for $Z^{t+1}_-$; the claim for $Z^{t+1}_+$ follows similarly.
In view of the recursion given in \prettyref{eq:recursionGamma},
$$Z^{t+1}_{ -}  = \frac{-d_+ + d_-}{2}  + \sum_{i=1}^{N_d } Y_i, $$
where $N_d \sim \Pois(d)$ is independent of $\{Y_i\}$; $\{Y_i\}$ are i.i.d.\
such that $ Y_i=F(Z_+^{t})$ with probability $ \rho b/d_- $ and $Y_i=F(Z_-^{t}  )$ with probability $\bar \rho c/ d_-$.
%Thus, $\expect{Z^{t+1}_{ -}  }= -d_+ + d_- + d \expect{Y_1}$ and $\var \left(Z^{t+1}_{ -}  \right) = d \expect{Y_1^2}$.
Since $F$ is monotone, $F(\infty)= \log (a/b)$, and $F(-\infty)=\log (b/c)$, it follows that
\begin{align*}
\expect{Y_1^2} & \ge  \min \{ |\log (a/b)|^2, | \log (b/c)|^2 \} =\Omega( b^{-1}) \\
\expect{|Y_1|^3} & \le \max\{ |\log (a/b)|^3, | \log (b/c)|^3 \} =O( b^{-3/2}).
\end{align*}
In view of  \prettyref{lmm:Poisson_BE}, we get that
\begin{align}
\sup_x \bigg|     \prob{  \frac{Z_-^{t+1} -\expect{Z_-^{t+1} } }{  \sqrt{ \var \left(Z^{t+1}_{ -}  \right)  } } \leq x } - \prob{Z \leq x}    \bigg|  =  O( b^{-1/2} ). \label{eq:gaussianconvergence}
\end{align}
It follows from \prettyref{lmm:meanvarianceZ} that
\begin{align*}
\expect{Z^{t+1}_{ -}  } & = - \theta - \lambda \expect{ \tanh(Z_+^t +  \varphi  ) }+ O( b^{-1/2} ). \\
\var \left(Z^{t+1}_{ -}  \right) & = \theta + \lambda \expect{ \tanh(Z_+^t + \varphi  ) }+ O( b^{-1/2} ).
\end{align*}
Using the area rule of expectation, we have that
\begin{align*}
& \expect{ \tanh(Z_+^t +  \varphi) } \\
& = \int_{0}^{1} \tanh'(t) \prob{ Z_+^t + \phi \ge t }  \diff t - \int_{-1}^{0} \tanh'(t)    \prob{ Z_+^t + \varphi \le t } \\
& = \int_{0}^{1} \tanh'(t) \prob{ v_t + \sqrt{v_t} Z +\varphi \ge t }  \diff t - \int_{-1}^{0} \tanh'(t)    \prob{ v_t + \sqrt{v_t} Z + \varphi \le t }  + O(b^{-1/2}) \\
& = \expect{ \tanh ( v_t+ \sqrt{v_t} Z + \varphi ) } +  O(b^{-1/2}).
\end{align*}
where the second equality follows from  the induction hypothesis and the fact that $|\tanh'(t)| \le 1$.
%By the induction hypothesis and the dominated convergence theorem, as $n \to \infty$,
%\begin{align*}
%\theta + \lambda  \expect{ \tanh(Z_+^t  + \varphi  ) } \to  \theta +  \expect{ \tanh(v_t + \sqrt{v_t} Z  + \varphi  ) }= v_{t+1},
%\end{align*}
%where the last equality holds due to the definition of $v_{t+1}$.
Recall that $ v_{t+1}=  \theta +  \lambda \expect{ \tanh(v_t + \sqrt{v_t} Z  + \varphi  ) }.$ Hence,
$\expect{Z_-^{t+1} } = -v_{t+1}+ O( b^{-1/2} )$ and $\var \left(Z^{t+1}_{ -}  \right) =v_{t+1}+ O( b^{-1/2} )$.
As a consequence, in view of \prettyref{eq:gaussianconvergence}, the desired \prettyref{eq:gaussianlimit}
holds for $Z^{t+1}_-$.
\end{proof}

\subsection{Proof of \prettyref{lmm:hmonotone}}

 By definition,
 \begin{align*}
 h(v) = \expect{ \tanh \left( v + \sqrt{v} Z + \varphi \right)   }.
 \end{align*}
 Since $|\tanh(x)| \le 1$, the continuity of $h$ follows from the dominated convergence theorem.
 We next show $h'(v)$ exists for $v \in (0, \infty)$.
 Notice that $\tanh'(x+ \sqrt{x} Z +\varphi) = ( 1- \tanh^2 ( x + \sqrt{x} Z +\varphi  )) ( 1+ x^{-1/2} Z/2 )$ for $x \in (0, \infty)$, and
 \begin{align*}
  \big| \left( 1- \tanh^2 ( x + \sqrt{x} Z + \varphi  ) \right) ( 1+ x^{-1/2} Z/2 ) \big| \le  1+ x^{-1/2} |Z|/2.
 \end{align*}
Since $|Z|$ is integrable, by the dominated convergence theorem, $\expect{\tanh'(x+ \sqrt{x} Z +\varphi)}$ exists and is continuous in $x$
over $(0, \infty)$.
Therefore, $x \to \expect{\tanh'(x+ \sqrt{x} Z+ \varphi)}$ is integrable over $x \in (0,\infty)$. It follows that
\begin{align*}
h(v)= \expect{ \tanh(\varphi) + \int_{0}^{v}  \tanh'(x + \sqrt{x} Z +\varphi ) \diff x } =\tanh(\varphi) + \int_0^v \expect{\tanh'(x + \sqrt{x} Z +\varphi )} \diff x,
\end{align*}
where the second equality holds due to Fubini's theorem. Hence,
 \begin{align*}
 h'(v) = \expect{ \left( 1- \tanh^2 ( v + \sqrt{v} Z+\varphi  ) \right) ( 1+ v^{-1/2} Z/2 ) }.
 \end{align*}
 Using the integration by parts, we can get that
 \begin{align*}
 & \expect{ \left( 1- \tanh^2 ( v + \sqrt{v} Z + \varphi  ) \right)   \sqrt{v} Z  } \\
  & =\int_{-\infty}^{\infty}  (1-\tanh^2( v + x +\varphi ) \frac{x }{\sqrt{2\pi v} }  \eexp^{-x^2/2v} \diff x \\
  & = -v  \int_{-\infty}^{\infty}  (1-\tanh^2( v+ x +\varphi )   \left( \frac{1}{\sqrt{2\pi v} } \eexp^{-x^2/2v} \right)'\diff x \\
  & = -v (1-\tanh^2( v+ x +\varphi )  \frac{1}{\sqrt{2\pi v} }  \eexp^{-x^2/2v} \bigg|_{-\infty}^{+\infty} + v
  \int_{-\infty}^{\infty}  (1-\tanh^2( v+ x +\varphi )'   \frac{1}{\sqrt{2\pi v} } \eexp^{-x^2/2v} \diff x \\
  &= -  2v\expect{\tanh(v+\sqrt{v} Z+\varphi) (1-\tanh^2( v+ \sqrt{v} Z +\varphi  ) ) }.
 \end{align*}
By combing the last two displayed equations, we get \prettyref{eq:hderivative}.

Next, we prove the concavity of $h$ in the special case with $\varphi=0$.
We will use the following equality coming from the change of measure: For $k \in \naturals$,
\begin{align*}
\expect{\tanh^{2k} ( \sqrt{v} Z  + v) } = \expect{\tanh^{2k-1} ( \sqrt{v} Z  + v) }.
\end{align*}
It follows from \prettyref{lmm:hmonotone} that
\begin{align*}
h'(v) & =  \expect{  \left( 1- \tanh ( v+ \sqrt{v} Z  ) \right) \left( 1- \tanh^2( v + \sqrt{v} Z )  \right) } \\
%& =  \frac{1}{4} \left( 1- 2 \expect{ \tanh^2 ( \sqrt{v} Z /2 + v/4 ) } + \expect{ \tanh^4 ( \sqrt{v} Z /2 + v/4 ) } \right) \\
& =  \expect{  \left( 1- \tanh^2( \sqrt{v} Z  + v ) \right)^2  } \\
& =   \expect{  \left(1- \tanh^2 \left(  \sqrt{v}  \left| Z  + \sqrt{v}  \right|  \right) \right)^2  },
\end{align*}
%Let $u \triangleq  \tanh ( \sqrt{v} Z /2 + v/4 )$. Applying the same argument as above, we can derive that
%\begin{align*}
%h''(v)&=\frac{1}{16} \left( \expect{ (-2 + 3 u^2) (1-u^2) } + \frac{1}{2} \expect{ (10 u -12 u^3 ) (1-u^2 ) } \right) \\
%&= \frac{1}{16} \expect{(-2+5u +3u^2 - 6 u^3 ) (1-u^2)} = \frac{1}{16} \expect{ -2 + 10 u - 14u^3 + 6 u^5 }.
%\end{align*}
where the last equality holds because $\tanh^2(x)$ is even in $x$. For $ 0<v<w<\infty$ and all $z \in \reals$,
\begin{align*}
\tanh^2 \left(  \sqrt{w}  \left| z  + \sqrt{v}  \right|  \right) \ge \tanh^2 \left(   \sqrt{v}  \left| z  + \sqrt{v}  \right|  \right).
\end{align*}
Thus,
\begin{align}
h'(v)   =  \expect{  \left(1- \tanh^2 \left(  \sqrt{v}  \left| Z  + \sqrt{v}  \right|  \right) \right)^2  }  \ge   \expect{  \left(1- \tanh^2 \left(  \sqrt{w}  \left| Z  + \sqrt{v}  \right|  \right) \right)^2  }. \label{eq:hv}
\end{align}
Let $X = \left| Z + \sqrt{v}  \right| $ and $Y= \left| Z  + \sqrt{w}  \right|$. Then for any $x \ge 0$,
\begin{align*}
\prob{ X \le x} = \prob{  - x - \sqrt{v} \le Z \le x - \sqrt{v}} \ge \prob{  -x - \sqrt{w} \le Z \le x - \sqrt{w} }  = \prob{Y \le x}.
\end{align*}
Hence, $X$ is first-order stochastically dominated by $Y$. Since $\left(1- \tanh^2 \left(  \sqrt{w}  x  \right) \right)^2$ is non-increasing in $x$ for $ x \ge 0$,
it follows that
\begin{align*}
 \expect{  \left(1- \tanh^2 \left(  \sqrt{w}  X  \right) \right)^2  } \ge  \expect{  \left(1- \tanh^2 \left(  \sqrt{w} Y  \right) \right)^2  }.
 \end{align*}
 Thus by \prettyref{eq:hv},
 \begin{align*}
 h'(v)  \ge \expect{  \left(1- \tanh^2 \left(  \sqrt{w}  \left| Z  + \sqrt{w}  \right|  \right) \right)^2  } =h'(w).
 \end{align*}

\subsection{Proof of \prettyref{lmm:treelowerbound_sym}}
The proof is very similar to the
proof of \prettyref{lmm:accuracyupperbound}; the key new challenge is that
$ \expect{ \calO(\sigma, \hat{\sigma}) }$ does not directly reduce to  the error probability of estimating $\sigma_u$ based on graph $G$.
We need a key lemma.
\begin{lemma}\label{lmm:correlationbound}
Fix any $t \ge 1$ and any two different vertices $i$ and $j$. For estimator $\hat{\sigma}(G)$ of $\sigma$ based on $G$,
\begin{align}
\expect{ \left(  \indc{ \hat{\sigma}_i \neq \sigma_i } - \frac{1}{2} \right)\left(  \indc{ \hat{\sigma}_i \neq \sigma_j } - \frac{1}{2} \right)  }\le \left( \frac{1}{2} - p^\ast_{T^t} \right)^2 + o(1). \label{eq:correlationbound}
\end{align}
\end{lemma}
\begin{proof}
Fix any $t \ge 1$. Recall that $G_u^t$ denotes the subgraph induced by vertices whose distance from $u$ is at most $t$
and $\partial G_u^t$ denotes the set of vertices whose distance from $u$ is precisely $t$.
Let $(T_i, \tau_{T_i})$ and $(T_j, \tau_{T_j})$ denote two independent copies of the tree model with $\rho=1/2$ and $a=c$ defined in
\prettyref{def:treemodel}.
The coupling lemma given in \prettyref{lmm:couplingtree} and \prettyref{lmm:asymptoticalIndependence} can be strengthened so that
there exists a sequence  of events $\calE_n$ such that $\prob{\calE_n} \to 1$ and on event $\calE_n$, $(G_i^t, \sigma_{G_i^t} )= (T_i^t, \tau_{T_i^t})$, $(G_j^t, \sigma_{G_j^t}) = (T_j^t, \tau_{T_j^t})$, and
\begin{align}
 \prob{\sigma_i, \sigma_j \big| G, \sigma_{\partial G_i^t}, \sigma_{\partial G_j^t} }  & =
 (1+o(1)) \prob{\sigma_i \big | G, \sigma_{\partial G_i^t}} \prob{ \sigma_j \big |G, \sigma_{\partial G_j^t} } \label{eq:independtwo_X}\\
  \prob{\sigma_i  \big | G, \sigma_{\partial G_i^t}}  & =
\left( 1+o(1) \right)   \prob{\sigma_i \big | G_i^t, \sigma_{\partial G_i^t}}  \label{eq:independtwo_XX} \\
  \prob{\sigma_j \big | G, \sigma_{\partial G_j^t}}  & =
\left( 1+o(1) \right)   \prob{\sigma_j \big | G_j^t, \sigma_{\partial G_j^t}}  \label{eq:independtwo_XXX} .
\end{align}
For $u=i, j$, define
\begin{align*}
X_u &= \prob{ \sigma_u =+ | G,  \sigma_{\partial G_u^t} } - \prob{ \sigma_u =-  | G,  \sigma_{\partial G_u^t } } \\
Y_u &= \prob{ \tau_u =+ | T_u^t,  \tau_{\partial T_u^t} } - \prob{ \tau_u =-  | T_u^t,  \tau_{\partial T_u^t } }.
\end{align*}
Then for any estimator $\hat{\sigma}(G)$, we have that
\begin{align*}
& \expect{ \left(  \indc{ \hat{\sigma}_i \neq \sigma_i } - 1/2 \right)\left(  \indc{ \hat{\sigma}_i \neq \sigma_j } - 1/2\right)  } \\
&=\expect{ \expect{ \left(  \indc{ \hat{\sigma}_i \neq \sigma_i } - 1/2\right)\left(  \indc{ \hat{\sigma}_i \neq \sigma_j } - 1/2\right)  \big| G,  \sigma_{\partial G_i^t}, \sigma_{\partial G_j^t} }  \indc{\calE_n} } + o(1) \\
&= \expect{ \expect{ \left(  \indc{ \hat{\sigma}_i \neq \sigma_i } -1/2 \right)  | G, \sigma_{\partial G_i^t} }
 \expect{ \left(  \indc{ \hat{\sigma}_j \neq \sigma_j } -1/2 \right) \big| G, \sigma_{\partial G_j^t} }
 \indc{\calE_n}} + o(1) \\
 & \le \frac{1}{4}
 \expect{ |X_i |  |X_j |  \indc{\calE_n}} + o(1) \\
 & = \frac{1}{4}
 \expect{ |Y_i|  |Y_j|  \indc{\calE_n}} + o(1)  = \frac{1}{4}  \expect{ |Y_i | |Y_j | } + o(1)  \\
 & =  \frac{1}{4}  \expect{ |Y_i | } \expect{ | Y_j  |} + o(1)
 = \left( 1/2 - p^\ast_{T^t} \right)^2 + o(1),
\end{align*}
where the first and fourth equality follows due to $\prob{\calE_n} \to 1$; the second equality
holds due to \prettyref{eq:independtwo_X}; the first inequality holds due to the fact that
$\prob{ \sigma_i  \neq x |G, \sigma_{\partial G_i^t}  }$ is maximized at $x=-$ if $X_i \ge 0$ and at $x= +$ if $X_i<0$;
the third inequality holds due to \prettyref{eq:independtwo_XX}, \prettyref{eq:independtwo_XXX},
$(G_i^t, \sigma_{G_i^t} )= (T_i^t, \tau_{T_i^t})$, and $(G_j^t, \sigma_{G_j^t}) = (T_j^t, \tau_{T_j^t})$ ; the firth equality follows
because $(T_i, \tau_{T_i})$ and $(T_j, \tau_{T_j} )$ are independent; the last equality follows because $p^\ast_{T^t} = 1/2-\expect{|Y_u|}/2$
by definition.  Hence we get the desired \prettyref{eq:correlationbound}.
\end{proof}
\begin{proof}[Proof of  \prettyref{lmm:treelowerbound_sym}]
Fix any estimator $\hat \sigma(G)$. Notice that by definition of  $\calO(\sigma, \hat{\sigma})$,
\begin{align*}
\calO(\sigma, \hat \sigma) = \frac{1}{2} - \bigg| \frac{1}{n} \sum_{i \in [n]} \indc{ \hat{\sigma}_i \neq \sigma_i } - \frac{1}{2} \bigg|.
\end{align*}
Therefore,
\begin{align}
\expect{\calO(\sigma, \hat \sigma) } & =\frac{1}{2} -  \expect{\bigg| \frac{1}{n} \sum_{i \in [n]} \indc{ \hat{\sigma}_i \neq \sigma_i } - \frac{1}{2} \bigg|}  \ge \frac{1}{2} - \expect{\left( \frac{1}{n} \sum_{i \in [n]} \indc{ \hat{\sigma}_i \neq \sigma_i } - \frac{1}{2} \right)^2}^{1/2},
\label{eq:overlaplowerbound}
\end{align}
where we used the Cauchy-Schwartz  in the last inequality.
Furthermore,
\begin{align*}
 \expect{\left( \frac{1}{n} \sum_{i \in [n]} \indc{ \hat{\sigma}_i \neq \sigma_i } - \frac{1}{2} \right)^2}
 & = \frac{1}{4n} + \frac{n-1}{n}
 \expect{ \left(  \indc{ \hat{\sigma}_1 \neq \sigma_1 } - \frac{1}{2} \right)\left(  \indc{ \hat{\sigma}_2 \neq \sigma_2 } - \frac{1}{2} \right)  },
\end{align*}
where we used the symmetry among vertices.
Applying \prettyref{lmm:correlationbound}, we get that
\begin{align*}
 \expect{\left( \frac{1}{n} \sum_{i \in [n]} \indc{ \hat{\sigma}_i \neq \sigma_i } - \frac{1}{2} \right)^2}
\le \left( \frac{1}{2} - p^\ast_{T^t} \right)^2 + o(1).
\end{align*}
Combining the last displayed equation with \prettyref{eq:overlaplowerbound} and noticing that $p^\ast_{T^t} \le 1/2$,
we get  the desired equality $\expect{\calO(\sigma, \hat \sigma) }  \ge  p^\ast_{T^t} + o(1)$. Since $\hat \sigma$ is arbitrary,
it follows that $\inf_{\hat{\sigma}} \expect{\calO(\sigma, \hat \sigma) }  \ge  p^\ast_{T^t} + o(1) $ and the proof is complete.
\end{proof}

\subsection{Proof of \prettyref{lmm:treeupperbound_sym}}
Before the main proof, we need a key lemma, which gives a recursive formula of $\tilde{\Gamma}^t_i$ on the tree model.
Its proof  is almost identical to the proof of  \prettyref{lmm:recursion} and thus omitted.
 \begin{lemma}\label{lmm:recursion_sym}
 For $t \ge 0$,
\begin{align}
\tilde{\Gamma}^{t+1}_{i }   =  \frac{1}{2} \;  \sum_{ j \in \partial i }   \log  \frac{ \exp \left(  2 \tilde{\Gamma}_j^t  \right) a +b } {  \exp \left( 2 \tilde{\Gamma}_j^t \right) b +a}.
\label{eq:recursion_Gamma_sym}
\end{align}
 with $\tilde{\Gamma}^0_i= \frac{1}{2} \log \frac{1-\alpha}{\alpha} $ if $\tilde{\tau}_i=+$ and $\tilde{\Gamma}^0_i= -\frac{1}{2} \log \frac{1-\alpha}{\alpha} $ if $\tilde{\tau}_i=-$.
  \end{lemma}

  %  % By the assumption $|\mu|>2$ and $b=n^{o(1/\log \log n) }$, it is shown by \cite[Theorem 2.3]{MNS:2013b} and
% \cite[Lemma 5.6]{MNS:2013a} that there is an estimator such that  for  any $u \in V\backslash U$, when we run it on
% the subgraph induced by vertices not in $G_u^{t-1}$ and $U$,
% its output satisfy $| W_u^{\ell} \Delta V^+ | \le \delta n $ for some $\ell \in \{\pm\}$
% and $0 \le \delta <1/2$ with probability converging to $1$; moreover, it takes in time $O(n \log^2n)$.

Let $V^{+}=\{ i \in V: \sigma_i=+\}$ and $V^{-}=\{ i \in V: \sigma_i=-\}$.
For an \ER random graph with edge probability $b/n$, it is well-known that if $b \to \infty$ and $b=o(\log n)$,
the maximum degree is at least $ \frac{\log n}{\log (\log n /b)}$ with high probability (see \cite[Appendix A]{HajekWuXuSDP14} for a proof).
Thus, with high probability, at least one vertex in $U$ has more than $ \frac{\log n}{2\log (\log n /b)}$
 neighbors in $V\backslash U$, so that $u_*$ is well-defined.
Due to the symmetry between $+$ and $-$,
without loss of generality, assume that $\sigma_{u_*} =+$.
By the assumption that $|\mu|>2$ and $b=n^{o(1/\log \log n)}$,
it is proved in \cite[Lemma 5.7]{MNS:2013a} that there exists an $\alpha \in (0,1/2)$
and a polynomial-time estimator such that for any $u \in V\backslash U$, when we apply the estimator in Step 3.1
of Algorithm \prettyref{alg:BPplusCorrelated},
 its output satisfies $ | W_u^{+} \Delta V^+ | \le \alpha  n $ and $ |W_u^{-} \Delta V^{-}| \le \alpha  n$
 after relabeling defined in Step 3.2 of Algorithm
 \prettyref{alg:BPplusCorrelated}.  Recall that $\alpha_u$ is the fraction of vertices misclassified by the partition $(W_u^+,  W_u^-)$.
 Thus, $\alpha_u \le \alpha$.

 Fix a vertex $u \in V\backslash U$. For all $i \in \partial G_u^t$,
 let $\tilde{\sigma}_i=+$ if $i \in W_u^{+}$
 and $\tilde{\sigma}_i=-$ if $i \in W_u^{-}$ after the labeling defined in Step 3.2 of Algorithm
 \prettyref{alg:BPplusCorrelated}.
 It is argued in \cite[Section 5.2]{MNS:2013a} that for each $i \in \partial G_u^t$, independently
 at random, $\tilde{\sigma}_i= \sigma_i$ with probability $1-\alpha_u$, and
$\tilde{\sigma}_i= - \sigma_i$ with probability $\alpha_u$.
Consider the tree model $(T_u, \tau)$ with $\rho=1/2$ and $a=c$,
where for each vertex $i \in T_u$, independently at random,
$\tilde{\tau}_i=\tau_i$ with probability $1-\alpha_u$
and $\tilde{\tau}_i= -\tau_i$ with probability $\alpha_u.$
By the coupling lemma given in \prettyref{lmm:couplingtree},
we can construct a coupling such that $(G_u^t, \sigma_{G_u^t}, \tilde{\sigma}_{\partial G_u^t} ) = (T_u^t, \tau_{T_u^t} , \tilde{\tau}_{\partial T_u^t} )$ holds with probability converging to $1$.
Moreover, on the event $(G_u^t, \sigma_{G_u^t},  \tilde{\sigma}_{\partial G_u^t}) = (T_u^t, \tau_{T^t},  \tilde{\tau}_{\partial T_u^t}  )$,
we have that  $R_u^t=\tilde{\Gamma}_u^t$ in view of the definition of $R_u^t$ given in Algorithm
 \prettyref{alg:BPplusCorrelated}, and the recursive formula of $\tilde{\Gamma}_u^t$ given in \prettyref{lmm:recursion_sym}.
Hence,
\begin{align*}
p_{G} ( \hat{\sigma}_{\rm BP}^t )  = \tilde{q}_{T^t, \alpha_u} +o(1),
\end{align*}
where the $o(1)$ term comes from the coupling error. Since  $\tilde{q}_{T^t, \alpha}$ is non-decreasing in $\alpha$,
it follows that
$$
p_{G} ( \hat{\sigma}_{\rm BP}^t )  \le \tilde{q}_{T^t, \alpha} +o(1).
$$
By the definition of $\calO(\hat{\sigma}, \sigma)$ given in \prettyref{eq:defoverlap},
\begin{align*}
\expect{ \calO(\hat{\sigma}_{\rm BP}, \sigma) } \le p_{G} ( \hat{\sigma}_{\rm BP}^t ),
\end{align*}
and the lemma follows by combining the last two displayed equations.

\subsection{Proof of \prettyref{lmm:gaussian_sym}}
Recall that $u_{t+1}=\frac{\mu^2}{4} \expect{ \tanh(u_{t} + \sqrt{u_t} Z ) }$ with $u_1=\frac{(1-2\alpha)^2 \mu^2}{4}$.
In the case $\rho=1/2$ and $\mu=\nu$,
$\theta=0$ and $\lambda = \frac{\mu^2}{4}$, and hence $u_t$ and $v_t$ satisfy the same recursion.
Also, comparing \prettyref{eq:recursion_Gamma_sym}
to \prettyref{eq:recursionGamma},
$\tilde{\Gamma}$ and $\Gamma$ satisfy the same recursion with $\rho =1/2$ and $\mu=\nu$.
Therefore, in view of the proof of \prettyref{lmm:gaussiandensityevolution}, to prove the lemma, it suffices to
show that in the base case with $t=1$,
\begin{align}
\sup_x \bigg|     \prob{  \frac{ \tilde{Z}_{\pm}^{ 1}  \mp u_{1}  }{  \sqrt{ u_1 } } \leq x } - \prob{Z \leq x}    \bigg|  = O(b^{-1/2}), \label{eq:gaussiansymmetryone}
 \end{align}
 Recall that $\tilde{\Gamma}^0_i= \frac{1}{2} \log \frac{1-\alpha}{\alpha} $ if $\tilde{\tau}_i=+$ and $\tilde{\Gamma}^0_i= -\frac{1}{2} \log \frac{1-\alpha}{\alpha} $ if $\tilde{\tau}_i=-$. Also, for all $i \in \partial u$, independently at random $\tilde{\tau}_i=\tau_i$ with probability $1-\alpha$,
 and  $\tilde{\tau}_i=\tau_i$ with probability $\alpha.$
 Let $ x^* = \frac{1}{2} \log \frac{a - \alpha(a-b) }{ b+ \alpha(a-b) }$, $d=\frac{a+b}{2}$
 and $\eta= \frac{b+(a-b)\alpha}{a+b}$.
 Thus, in view of the recursion given in \prettyref{eq:recursion_Gamma_sym},
$
 \tilde{\Gamma}^1_u = \sum_{i=1}^{N_d} X_i,
$
where conditional on $\tau_u=\pm$, $N_d \sim \Pois(d)$ is independent of $\{X_i\}$; $\{X_i\}$ are i.i.d.\ such that
conditional on $\tau_u=+$, $ X_i= x^* $ with probability $1-\eta$ and $X_i=-x^* $ with probability $\eta$;
conditional on $\tau_u=-$, $ X_i=x^* $ with probability $\eta$ and $X_i=-x^*$ with probability $1-\eta$.
Taylor expansion yields that
\begin{align*}
x^*= \frac{(1-2\alpha) (a-b) }{2b} + O(b^{-1}).
\end{align*}
By conditioning the label of $u$ is $-$, it follows that
\begin{align*}
\expect{\tilde{Z}_-^1} & = - d (1-2\eta) x  = -\frac{(1-2\alpha)(a-b)}{2} x =
-\frac{(1-2\alpha)^2 (a-b)^2}{4b} +  O(b^{-1/2})  \\
& =  - u_1 +O(b^{-1/2}), \\
\var\left( \tilde{Z}_-^1 \right) &= d x^2 = \frac{(1-2\alpha)^2 (a-b)^2}{4b} +O(b^{-1/2})=u_1 +O(b^{-1/2}) .
\end{align*}
In view of \prettyref{lmm:Poisson_BE},
we get that
$$
\sup_x \bigg|     \prob{  \frac{  \tilde{Z}_{\pm}^{ 1}  - \expect{ \tilde{Z}_{\pm}^{ 1}   } }{  \sqrt{ \var \left(  \tilde{Z}_{\pm}^{ 1}  \right)  }  } \leq x} -  \prob{Z \leq x}   \bigg|  \leq  O (b^{-1/2} ).
$$
Hence, we proved \prettyref{eq:gaussiansymmetryone} for $\tilde{Z}_-^1$. By symmetry between $-$ and $+$,
the desired \prettyref{eq:gaussiansymmetryone} also holds for $\tilde{Z}_+^1$.

\section{A Data-driven Choice of the Parameter $\alpha$ in Algorithm \prettyref{alg:BPplusCorrelated} } \label{app:estimateofalpha}

\begin{algorithm}[htb]
\caption{Estimation of $\alpha_u$}\label{alg:estimationalpha}
\begin{algorithmic}[1]
\STATE Take $U \subset V$ to be a random subset of size $\lfloor \sqrt{n} \rfloor $ and $S  \subset V$
to be a random subset of size $\lfloor n/\log b \rfloor$.
Let $u_* \in U$ be a random vertex in $U$
with at least $\frac{\log n}{2\log (\log n/b) }$neighbors in $V\backslash U\backslash S.$
\STATE For $u \in V\backslash U \backslash S$ do
\begin{enumerate}
\item Run a polynomial-time estimator capable of correlated recovery on the subgraph induced by vertices not in $G_u^{t-1}$ and
$U \cup S$, and let $W_u^+$ and $W_u^-$ denote the output of the partition.
\item Relabel $W_u^+$ and $W_u^-$ such that if $a>b$, then $u_*$ has more neighbors in $W_u^+$ than $W_u^-$; otherwise,
$u_*$ has more neighbors in $W_u^-$ than $W_u^+$.
\item Take $T \subset W_u^+ \cup W_u^-$ to be a random subset of size  $\lfloor \sqrt{n} \rfloor $. Let $T_0\subset T$ denote the set of vertices with at least $\frac{\log n}{2\log (\log n /b)  }$ neighbors in $S$. Let $T_1$ denote a random subset of
$T_0$ with size $\lfloor \frac{\log n }{b} \rfloor$.
\item Run a polynomial-time estimator capable of correlated recovery on the subgraph induced by vertices not in $G_u^{t-1}$ and
$U \cup T$.  Let $W^{+}$
and $W^{-}$ denote the output of the partition. Relabel $(W^+, W^-)$ in the same way as $(W_u^+, W_u^-)$.
\item Let $T_1^+$ consists of vertices $i \in T_1$ with more neighbors in $W^+$ than $W^-$; let $T_1^-=T_1\backslash T_1^+$.
Define $\hat{\alpha}_u= \frac{ | T_1^+ \cap W_u^- | + | T_1^- \cap W_u^+ | }{|T_1 |}.$
\end{enumerate}
\end{algorithmic}
\end{algorithm}

Algorithm \prettyref{alg:BPplusCorrelated} requires the knowledge of $\alpha_u$, which is given by $
\alpha_u=| W_u^{+} \Delta V^+ |/n = |W_u^{-} \Delta V^{-}|/n.
$
In this section, we show that there exists an efficient estimator $\hat{\alpha}_u$ such that $\hat{\alpha}_u=\alpha_u+o(1)$
with high probability. Our estimation procedure is given in Algorithm \prettyref{alg:estimationalpha}.

%\begin{enumerate}
%\item Let  $T$ be a random subset  of size $\lfloor \sqrt{n} \rfloor$. Let $T_0\subset T$ denote the set of vertices
%with at least $\frac{\log n}{2\log (\log n /b)  }$ neighbors in $V\backslash T$. Let $T_1$ denote a random subset of
%$T_0$ with size $\lfloor \frac{\log n }{b} \rfloor$.
%\item Run an estimator capable of correlated recovery on the subgraph induced by vertices in $V\backslash T \backslash U$.
% Let $W^{+}$
%and $W^{-}$ denote the output of the partition.
%\item Let $T_1^+$ consists of vertices $i \in T_1$ with more neighbors in $W^+$ than $W^-$; let $T_1^-=T_1\backslash T_1^+$.
%\item In Step 3.3 of Algorithm \prettyref{alg:BPplusCorrelated}, define $\hat{\alpha}_u$ as \nb{JX: need to coordinate the sign
%of $W$ and $W_u$}
%\begin{align*}
%\hat{\alpha}_u = \frac{ | T_1^+ \cap W_u^- | + | T_1^- \cap W_u^+ | }{|T_1 \cap (W_u^+ \cup W_u^-) |}.
%\end{align*}
%\end{enumerate}

\begin{lemma}
Let $\hat{\alpha}_u$ be the output of Algorithm \prettyref{alg:estimationalpha}. Then with probability converging to $1$, $\hat{\alpha}_u=\alpha_u+o(1)$.
\end{lemma}
\begin{proof}
We assume $a>b$ in the proof; the case $a<b$ can be proved similarly.
Let $k^\ast=  \frac{\log n}{2\log (\log n /b)  }$.
For any vertex $i$ in $T$, let $d_i$ denote its number of neighbors in $S$.
Then $d_i$ is stochastically lower bounded by $\Binom( |S|, b/n)$.
Since $|S| =\lfloor n/\log b \rfloor$, it follows that (see \cite[Appendix A]{HajekWuXuSDP14} for a proof)
\begin{align*}
- \log \prob{d_i \ge k^\ast } \le \frac{1}{2} \log n - \log \log n.
\end{align*}
Because $\{d_i\}_{i \in T}$ are independent,  the cardinality of set $T_0$ is stochastically lower bounded by
$\Binom( \lfloor \sqrt{n} \rfloor, \log n/\sqrt{n})$. Therefore, with high probability $|T_0| \ge \frac{1}{2} \log n $
and thus $T_1$ is well-defined.

Define the event $E_1$  to be that $T_1^+= T_1 \cap V^+$ and $T_1^- = T_1 \cap V^-$.
We claim that $\prob{E_1} \to 1$.
In fact, fix any vertex $i \in T_1$, suppose $\sigma_i=+$
without loss of generality and let $\calN(i)$ denote the set of its neighbors in $V \backslash T$.
Let
$$
 | W^{+} \Delta V^+ |/n \le \delta n, \quad   |W^{-} \Delta V^{-} | \le \delta n.
$$
Then $\delta \in [0,1/2)$.
Notice that $d_i$ is independent of the partition $W^{+}$ and $W^{-}$. Thus, conditional on $d_i$,
$|\calN(i) \cap W^+ |$ is stochastically lower bounded by $ \Binom( d_i, \frac{a}{a+b}  - \frac{(a-b) \delta} {a+b} )$
and $ | \calN(i) \cap W^-| $ is stochastically upper bounded by $ \Binom( d_i,  \frac{b}{a+b}  + \frac{(a-b) \delta} {a+b} )$.
It follows from the Chernoff bound that conditional on $d_i$,
\begin{align*}
\prob{  |\calN(i) \cap W^+ |  <  \frac{(a-b) d_i}{2(a+b)} } & \le  \eexp^{-\Omega( d_i (a-b)^2/a^2 ) },  \\
\prob{  |\calN(i) \cap W^- |  > \frac{(a-b) d_i}{2(a+b)} } & \le  \eexp^{-\Omega( d_i (a-b)^2/a^2 ) }.
\end{align*}
Due to $d_i \ge k^\ast$ for all $i \in T_0$, it yields that
\begin{align*}
\prob{  |\calN(i) \cap W^+ | < |\calN(i) \cap W^- | } \le 2 \eexp^{-\Omega( k^\ast (a-b)^2/a^2 ) }.
\end{align*}
Applying the union bound, we get that
\begin{align*}
\prob{ \exists i \in T_1 \cap V^+,   |\calN(i) \cap W^+ | < |\calN(i) \cap W^- | }  \le \frac{2\log n}{b} \eexp^{-\Omega( k^\ast (a-b)^2/a^2 ) }.
\end{align*}
By the assumption that $a-b/\sqrt{a}=O(1)$ and $b=o(\log n)$, it follows that
\begin{align*}
 \frac{\log n}{b} \eexp^{-\Omega( k^\ast (a-b)^2/a^2 ) } \le \frac{\log n}{b} \exp \left(  -  \Omega\left(\frac{\log n}{b \log (\log n /b) } \right)  \right) =o(1),
\end{align*}
Combing the last two displayed equations gives that with high probability, for all $i \in T_1 \cap V^+$,  $|\calN(i) \cap W^+ | > |\calN(i) \cap W^- | $
and thus $i \in T_1^+$. Similarly, one can show that with high probability, for all $i \in T_1 \cap V^-$, $|\calN(i) \cap W^+ | <  |\calN(i) \cap W^- | $ and thus $i \in T_1^-$.
Hence, $\prob{E_1} \to 1$.

Finally, we show $\hat{\alpha}_u=\alpha_u+o(1)$ with high probability.
Let $W_u=W_u^+ \cup W_u^-$.
Notice that $T_1$ is randomly chosen and independent of the partition $(W_u^+, W_u^-)$.
Hence for $i \in T_1 $, it lies in $(W_u^- \cap V^+) \cup (W_u^+ \cap V^-) $ with probability $\alpha_u$.
Therefore,
\begin{align*}
|T_1 \cap V^+ \cap W_u^-| + |T_1 \cap V^- \cap W_u^+ |  \sim \Binom( |T_1 |, \alpha_u).
\end{align*}
Define the event $E_2$ to be
\begin{align*}
\frac{|T_1 \cap V^+ \cap W_u^-| + |T_1 \cap V^- \cap W_u^+ | }{ |T_1 | }  =  \alpha_u \left(1+ (b/\log n )^{1/4} \right).
\end{align*}
Then it follows from the Chernoff bounds that $\prob{E_2} \to 1$. By the union bound, $\prob{E_1 \cap E_2} \to 1$.
Notice that on the event $E_1 \cap E_2$,
\begin{align*}
\hat{\alpha}_u = \frac{|T_1 \cap V^+ \cap W_u^-| + |T_1 \cap V^- \cap W_u^+ | }{ |T_1| }  =  \alpha_u \left(1+ (b/\log n)^{1/4} \right).
\end{align*}
Therefore, we conclude that $\hat{\alpha}_u=\alpha_u+o(1)$ with high probability.
\end{proof}
\end{document}